\documentclass[11pt]{article}

\date{}

\usepackage[utf8]{inputenc} % allow utf-8 input
\usepackage[T1]{fontenc}    % use 8-bit T1 fonts
\usepackage{url}            % simple URL typesetting
\usepackage{booktabs}       % professional-quality tables
\usepackage{amsfonts}       % blackboard math symbols
\usepackage{nicefrac}       % compact symbols for 1/2, etc.
\usepackage{microtype,nicefrac}      % microtypography
\usepackage{xspace}
\usepackage{natbib}

\usepackage{comment}
\usepackage{amsmath}
\usepackage{amsthm}
\usepackage{graphicx}
\usepackage{rotating}
\usepackage{authblk}

\usepackage{amssymb}
\usepackage{autobreak}
\usepackage{mathtools}
\usepackage[dvipsnames]{xcolor}
\usepackage{bbm}
\usepackage[ruled,vlined]{algorithm2e}
\usepackage{algorithmic}

\usepackage{tikz}
\usetikzlibrary{positioning}
\usetikzlibrary{trees}

\usetikzlibrary{automata, positioning, arrows}
\tikzset{
->,% makes the edges directed
>=stealth, % makes the arrow heads bold
node distance=2cm, % specifies the minimum distance between two nodes. Change if necessary.
every state/.style={thick, fill=gray!10}, % sets the properties for each ’state’ node
initial text=$ $, % sets the text that appears on the start arrow
}
\usepackage{subcaption}
\usepackage[colorlinks=true,linkcolor=blue,allcolors=blue]{hyperref}       % hyperlinks

\makeatletter

\DeclareMathOperator*{\argmax}{arg\,max}
\DeclareMathOperator*{\argmin}{arg\,min}

\newcommand{\E}{\mathbb{E}}
\newcommand{\R}{\mathbb{R}}

\newcommand{\Z}{\mathcal{Z}}
\newcommand{\A}{\mathcal{A}}

\newcommand{\Ss}{\mathcal{S}}
\newcommand{\F}{\mathcal{F}}
\newcommand{\U}{\mathcal{U}}
\newcommand{\cvar}{\text{CVaR}}
\newcommand{\cvara}{\text{CVaR}_{\alpha}}

\newtheorem{theorem}{Theorem}[section]
\newtheorem{corollary}[theorem]{Corollary}
\newtheorem{lemma}[theorem]{Lemma}
\newtheorem{proposition}[theorem]{Proposition}
\theoremstyle{remark}
\newtheorem{remark}[theorem]{Remark}
\theoremstyle{plain}
\newtheorem{assumption}[theorem]{Assumption}

\newtheoremstyle{named}{}{}{\itshape}{}{\bfseries}{.}{.5em}{\thmnote{#3}#1}
\theoremstyle{named}

\theoremstyle{definition}
\newtheorem{definition}{Definition}

\theoremstyle{definition}

\title{On the Convergence and Optimality of Policy Gradient for Markov Coherent Risk}

\author[1]{Audrey Huang%\footnote{corresponding author: \texttt{audreyh@andrew.cmu.edu}.}
}
\author[1]{Liu Leqi}
\author[1]{Zachary C. Lipton}
\author[2]{Kamyar Azizzadenesheli}
\affil[ ]{\texttt{\href{mailto:audreyh@andrew.cmu.edu}{\textcolor{black}{audreyh@andrew.cmu.edu}},\href{mailto:leqil@cs.cmu.com}{\textcolor{black}{leqil@cs.cmu.edu}}, \href{mailto:zlipton@cmu.edu}{\textcolor{black}{zlipton@cmu.edu}},\href{mailto:kamyar@purdue.edu}{\textcolor{black}{kamyar@purdue.edu}}}}
\affil[1]{Machine Learning Department, Carnegie Mellon University}
\affil[2]{Department of Computer Science, Purdue University}

\begin{document}

\maketitle

\begin{abstract}
In order to model risk aversion in reinforcement learning,
an emerging line of research adapts familiar algorithms
to optimize \emph{coherent} risk functionals,
a class that includes conditional value-at-risk (CVaR).
Because optimizing the coherent risk 
is difficult in Markov decision processes,
recent work tends to focus
on the Markov coherent risk (MCR),
a time-consistent surrogate. 
While, policy gradient (PG) updates 
have been derived for this objective, 
it remains unclear
(i) whether PG finds a global optimum for MCR;
(ii) how to estimate the gradient in a tractable manner.
In this paper, we demonstrate that,
in general, MCR objectives
(unlike the expected return)
are not gradient dominated
and that stationary points 
are not, in general,
guaranteed to be globally optimal.
Moreover, we present a tight upper bound 
on the suboptimality of the learned policy,
characterizing its dependence 
on the nonlinearity of the objective
and the degree of risk aversion.
Addressing (ii), we propose 
a practical implementation of PG 
that uses state distribution reweighting 
to overcome previous limitations.
Through experiments, we demonstrate  
that when the optimality gap is small, 
PG can learn risk-sensitive policies.
However, we find that instances 
with large suboptimality gaps are abundant
and easy to construct, 
outlining an important challenge for future research.

\end{abstract}
\section{Introduction}
As reinforcement learning (RL) has emerged
as a prospective technology
in consequential and safety-critical domains,
a burgeoning body of research  
seeks to optimize objectives 
that incorporate notions of risk sensitivity~\citep{garcia2015comprehensive},
e.g., reducing variance 
or mitigating the worst case returns.
In particular, \emph{coherent} risk functionals~\citep{artzner1999coherent}
satisfy several natural and desirable properties,
subsuming the expected return,
conditional value-at-risk (CVaR),
mean upper semideviation, 
and spectral risk functionals~\citep{acerbi2002spectral}. 
However, when directly addressing the coherent risk 
in Markov decision process (MDPs),
owing to time inconsistency,
it can be difficult to construct 
% and solve 
dynamic programming formulations 
or to derive the corresponding Bellman operators.
Thus many researchers instead focus
on the \emph{Markov coherent risk} (MCR),
a time-consistent surrogate 
\citep{ruszczynski2006optimization, tamar2015policy}.

While recent papers derive policy gradient (PG) updates for MCR~\citep{tamar2015policy}, 
% the policy gradient (PG) learning algorithm for MCR
PG for this objective
has not been proven to enjoy the same guarantees 
%of convergence to global optima 
of global convergence
as it has for the expected return
\citep{agarwal2019theory,bhandari2019global}.
Moreover, while folk wisdom suggests that
%, even here, PG 
it
should converge to a stationary point of the MCR, 
this local convergence 
viewpoint 
leaves open the critical question of 
how far this point can be from the global optimum. 
%how far the solution can be from the global optimum.

In this paper, we analyze the global convergence 
of PG algorithms for optimizing the MCR, 
which, to the best of our knowledge, 
is the first investigation into its theoretical convergence. 
% applied to the MCR,
% investigating whether they are gradient dominated,
% and whether they converge to global optima. 
% Markov coherent risk. 
% To the best of our knowledge, 
% this is the first investigation 
% into the theoretical convergence 
% of policy gradient for MCR.
First, we prove that the MCR 
is not gradient dominated 
and thus, in general,
% policy gradient  
PG may not converge 
to a global optimum
(Section~\ref{SEC:GRADIENT_DOMINATION}). 
We upper bound the suboptimality of the policy 
using a scalar notion of first-order stationarity
plus necessary problem-dependent residuals 
reflecting the nonlinearity 
and risk sensitivity of the objective.
Moreover, we design MDP problem instances 
that illustrate how these residuals are inevitable,
and thus that our upper bound is tight.
% Next, we demonstrate that, 
% when the residuals are small, 
% projected gradient descent 
% with a directly parameterized policy 
% converges in $O(1/\sqrt{T})$ iterations 
% (Section \ref{sec:pgd}). 
Next, we show that projected gradient descent of the MCR objective converges in $O(1/\sqrt{T})$ iterations to a policy with suboptimality gap quantified by the same residuals (Section \ref{sec:pgd}). We later show that the mentioned suboptimality gap vanishes under the expected return, which is an MCR objective, recovering the convergence results of \citet{agarwal2019theory}. 

%  \leqi{Kamyar moved this as a separate paragraph, perhaps we could move it to the beggining of the previous paragraph? And also add a sentence about the performance difference lemma result.}

On the practical side, 
% We also provide methods 
we provide methods
to estimate the gradient of the MCR,
% which is difficult because 
a task that proves difficult
because 
it involves an expectation over 
a reweighted state distribution of the underlying MDP 
(Theorem~\ref{thm:mcrp_gradient}).
The weights are determined by the solution
%where the weights are the solution 
to the maximization problem in the MCR Bellman operator (Section \ref{SEC:VALUE}), unique 
to the coherent risk functional 
and dependent on the current policy. 
In \citet{tamar2015policy}, the authors 
make the strong assumption 
% of being able to 
that the agent can
sample from the reweighted distribution of the underlying MDP, 
which rarely holds in online settings. 
To relax this assumption, 
in the second part of our paper, 
we leverage recent advances 
in off-policy evaluation 
to propose an algorithm using 
state distribution correction~\citep{liu2018breaking} 
to tractably estimate the gradient, 
which can then be plugged 
into any policy gradient algorithm. 
This permits, to our knowledge, 
the first evaluations of 
% policy gradient optimization 
PG optimization
of MCR in MDPs.
% \leqi{check out the sentence ahead.}  
% Our method can be plugged 
% into any 
% % policy gradient
% PG 
% algorithm.
% 
% and we combine it with an
In our experiments (Section \ref{sec:experiments}), we integrate this method
into an actor-critic algorithm 
to demonstrate its effectiveness in optimizing MCR 
in a stochastic version of the Cliffwalk environment \citep{sutton2018reinforcement}.

\section{Related Work} 
Risk functionals, including 
exponential utility~\citep{fei2020risksensitive}, 
the mean-variance risk functional~\citep{di2012policy, prashanth2013actor}, 
the conditional value-at-risk (CVaR)~\citep{tamar2014optimizing, chow2014algorithms}
and cumulative prospect risks~\citep{prashanth2016cumulative}
are increasingly studied in the RL literature.
They have been explored empirically 
for a number of Atari games
\citep{bellemare2017distributional}, 
and in the context of such diverse applications 
as autonomous driving \citep{mavrin2019distributional} 
and healthcare \citep{keramati2019being}.
Many (but not all) of these risk functionals
are subsumed under the class of coherent risks
\citep{shapiro2009lectures},
which satisfy several desirable theoretical properties
(Section \ref{sec:preliminaries}). 

In the MDP setting, researchers have considered 
the dynamic coherent risk~\citep{chow2014algorithms},
which applies the coherent risk functional 
directly on the (random) return over entire trajectories
and the MCR objective \citep{tamar2015policy},
which applies it in a nested manner at each timestep.
\citet{pflug2016time} shows 
that the two formulations 
are generally not equal 
except for the expected return and max-risk
due to a property called \emph{time consistency}. 
Time consistent risk measures satisfy the property 
that if a policy is risk-optimal 
for $T$ timesteps onward, 
it is also risk-optimal for $t$ timesteps onward, 
where $t \leq T$.
Leveraging time consistency, 
one can rewrite MCR objectives 
in dynamic programming form 
using a Bellman operator, 
as derived by \citet{ruszczynski2006optimization} 
and studied in the PG setting 
by \citet{tamar2015policy}. 
By contrast, 
for the dynamic coherent risk,
the Bellman optimality operator
has been derived only for CVaR
\citep{chow2015risk}.
% and is much more difficult to solve for a number of reasons---it 
Even for CVaR, the dynamic coherent risk 
is difficult to work with,
requiring that we augment the state space 
with the continuous CVaR level $\alpha$ 
($\text{CVaR}_\alpha(Z)$ 
gives the expected value 
of the lower $\alpha$-quantile 
of the random variable $Z$), 
and learn the optimal value for all CVaR levels. 
The MCR upper bounds 
the dynamic coherent risk,
and though not always explicitly named, 
has seen wide usage in RL literature
\citep{tamar2014optimizing, bellemare2017distributional}.%keremati2020

In the primary work on PG for MCR objectives, \citet{tamar2015policy} studies 
the general class of coherent risk functionals,
introducing a (statistically) consistent
sampling-based estimate of the policy gradient.
However, they do not investigate convergence 
of the PG algorithm to local or global optima. 
Further, while they present 
an actor-critic algorithm 
for optimizing the MCR, 
they make strong sampling assumptions. 

Our work makes advances on both of these fronts. 
Inspired by
\citet{agarwal2019theory,kamyar,bhandari2019global},
who demonstrate the global optimality 
and convergence rate 
of PG for the expected return,
our convergence analysis characterizes
the suboptimality of PG for MCR.
Our proposed policy gradient algorithm 
is inspired by the algorithm 
presented in \citet{tamar2015policy},
utilizing state distribution correction methods 
first proposed by \citet{liu2018breaking} 
for off-policy evaluation 
and later adapted in \citet{liu2020off} 
for off-policy optimization.

\section{Preliminaries}\label{sec:preliminaries}
Let $(\Omega, \mathcal{F}, P)$ denote a probability space where $\Omega$ is a finite set of outcomes, 
$\mathcal{F}$ is the $\sigma$-algebra over $\Omega$, 
and $P \in \mathcal{B}$ is a probability measure over $\F$ and $\mathcal{B} = \left\{\xi: \int_{\omega \in \Omega}\xi(\omega) = 1, \xi(\omega) \geq 0\right\}$ is a set of probability measures.
Let $\mathcal{Z}$ be the space of real-valued random variables $Z:\Omega\rightarrow \R$ defined over $(\Omega, \F, P)$. 
A \emph{risk functional} $\rho: \Z \rightarrow \overline{\R}$ maps a random variable $Z$ 
to a value on the extended real line $\overline{\R} := \R \cup \{-\infty, \infty\}$. 
We call $\rho$ \emph{coherent} if it satisfies the following four properties \citep{artzner1999coherent}:
\begin{enumerate}
    \item Monotonicity: $\rho(Z_1) \leq \rho(Z_2)$ whenever $Z_1 \leq Z_2$ almost surely; 
    \item Convexity: $\rho(tZ_1 + (1-t)Z_2) \leq t\rho(Z_1) + (1-t)\rho(Z_2)$ for $t \in [0, 1]$;
    \item Translation equivariance: $\rho(Z + c) = \rho(Z) + c, \forall c \in \R$;
    \item Positive homogeneity: $\rho(tZ) = t\rho(Z)$ for $t > 0$.
\end{enumerate}
% Popular coherent risk functionals,
% including expected value, 
% conditional value-at-risk (CVaR), 
% and mean upper semideviation are coherent.
Importantly, coherent risk functionals 
can be uniquely expressed 
using the following dual representation:
\begin{theorem}
\citep[Theorem 6.4]{shapiro2009lectures}
\label{thm:dual}
A risk functional $\rho: \mathcal{Z} \to \overline{\R}$ 
is coherent if and only if there exists 
a convex bounded and closed set of \emph{feasible duals} $\U$ 
depending on $P$, such that 
\begin{equation}\label{eq:dualform}
    \rho(Z) = \sup_{\xi \in \U(P)}\E[\xi Z],%\kamyar{\text{I remember it was $\sup$ in the book}}
\end{equation}
where the risk envelope $\mathcal{U}(P) = \{\xi \in \text{dom}(\rho^*)%\partial \rho(0)
: \int_{\omega \in \Omega} \xi(\omega) dP(\omega) = 1, \xi \geq 0\}$ and 
$\text{dom}(\rho^*)$ is the domain of the Fenchel conjugate of $\rho$. 
%$\partial \rho(0)$ is the subdifferential of $\rho$ at $Z=0$. 
%($Z(\omega) = 0$ for all $\omega \in \Omega$). 
\end{theorem}
Because the set of subdifferentials is always a convex and compact set, following \cite{tamar2015policy}, we write $\U$ using the general form below, which includes additional assumptions on smoothness: 
\begin{assumption}[General Form of Risk Envelope]\label{assum:risk_envelope}
The risk envelope $\U(P)$ from Theorem \ref{thm:dual} can be written as 
\begin{align} \label{eq:risk-envelop}
    \U(P) = \{&\xi : g_e(\xi, P) = 0 \quad\forall e \in \mathcal{E},  f_i(\xi, P) \leq 0 \quad\forall i \in \mathcal{I}, 
    \xi \geq 0 , \E_{P}[\xi] = 1 \},
\end{align}
%\begin{align}
%   \U(P) = \{&\xi : g_e(\xi, P) = 0 \quad\forall e \in \mathcal{E},  \label{eq:risk-envelop} \\
%   &f_i(\xi, P) \leq 0 \quad\forall i \in \mathcal{I}, 
%    \xi \geq 0 , \E_{P}[\xi] = 1 \},\nonumber
%\end{align}
where $\mathcal{E}, \mathcal{I}$ denote 
a set of equality and inequality constraints, respectively. 
Each constraint $g_e$ is affine in $\xi$, 
each constraint $f_i$ is convex in $\xi$, 
and there exists a strictly feasible point $\overline{\xi}$.
Further, for any given $\xi$, 
both $g_e(\xi, P)$ and $f_i(\xi, P)$ 
are twice differentiable in $P$ 
and there exists $M > 0$ such that 
for all $e \in \mathcal{E}$ and $i \in \mathcal{I}$, 
{%\small
$$\Big\Vert \frac{dg_e(\xi, P)}{dP}  \Big\Vert_{op} \leq M \quad\text{and}\quad \Big\Vert \frac{df_i(\xi, P)}{dP}  \Big\Vert_{op} \leq M.$$
}%
\end{assumption}
Assumption \ref{assum:risk_envelope} is satisfied 
by the CVaR, mean semideviation, and spectral risk functionals,
and implies that the risk envelope $\U$ is known in explicit form.  
The dual representation in Theorem \ref{thm:dual} 
%can thus be interpreted as suggesting 
indicates
that coherent risk measures can be seen 
as the eifjcctrdujltgunddiuvvjvffvldfukdnueiejjrtik
worst-case expectation of $Z$ 
over probability distributions reweighted by $\xi$,
which must be chosen from the set of measures 
in the risk envelope $\U$~\citep{chow2015risk}.

\paragraph{Conditional Value-at-Risk (CVaR)}
CVaR, which arises frequently
in the portfolio optimization
and (more recently) RL literatures,
corresponds to the expected return over 
the worst $\alpha$ fraction of outcomes.
Formally, the CVaR at a level $\alpha \in [0, 1]$ 
of a random variable $Z$ is defined as 
$\rho_{\text{CVaR}_\alpha} = \inf_{t \in \R}\left\{t + \frac{1}{\alpha}\E[(Z - t)_+]\right\}$. 
It can equivalently be expressed 
in the dual formulation (Theorem \ref{thm:dual}) 
using the risk envelope 
$\U_{\text{CVaR}_\alpha}(P) = \left\{\xi : \xi \in \Big[0, \frac{1}{\alpha}\Big], \E_{P}[\xi] = 1 \right\}$, 
for which
the solution $\xi^*$ to the maximization problem in~\eqref{eq:dualform} 
can be calculated in closed form 
as $\xi^*(\omega) = \frac{1}{\alpha}$ when $Z(\omega) > \lambda^{*}$ 
and as $\xi^*(\omega) = 0$ when $Z(\omega) < \lambda^{*}$, 
where $\lambda^{*}$ is any $(1-\alpha)$-quantile of $Z$
\citep[Chapter 6]{shapiro2009lectures}.

\section{Problem Setting}\label{sec:problem_setting}
In this paper, we focus on the infinite horizon discounted 
Markov Decision Process (MDP) setting. 
An MDP is a tuple $\mathcal{M} = (\Ss, \A, C, P, \gamma, s_0)$,
where $\Ss$ is the state space, 
$\A$ is the action space, 
$C:\Ss\rightarrow[0, 1]$ is the cost function,  
% is \kamyar{maybe, simply just say $C:\Ss\rightarrow[0, C_{\text{max}}]$ is the cost function, no need to say " a per-state deterministic  and bounded cost,"} a per-state deterministic  and bounded cost,
$P(\cdot|s,a)$ is the transition kernel, $\gamma \in [0, 1)$ is the discount factor, and $s_0 \in \Ss$ is the starting state. 
%Actions are chosen according to 
A stationary Markov policy 
$\pi_\theta : \Ss \rightarrow \Delta(\A)$ 
parameterized by $\theta \in \R^d$ 
maps each state $s \in \Ss$ 
to a probability measure over actions $\A$, where $\Delta(\cdot)$ denotes the probability simplex. 

Given an MDP $\mathcal{M}$ and policy $\pi_\theta$, consider a filtration $\F=\{\mathcal{F}_t\}_{t\geq0}$ such that the observations sequence $(s_0, a_0, c_0, \ldots, s_t, a_t, c_t)$ is $\mathcal{F}_t$-measurable. 
%Let $Z_t := C(s_t)$ 
Let $C(s_t)$
be the $\F_t$-measurable random cost at time $t$. 
In this paper, we consider the 
\emph{Markov coherent risk} (MCR) objective 
defined over $\{C(s_t)\}_{t \geq 0}$.  
% defined over $\{Z_t\}_{t \geq 0}$.  

% Given an MDP $\mathcal{M}$ and policy $\pi_\theta$, for time $t \geq 0$ let the time-dependent random variable $Z_t:= C(s_t)$ be the cost obtained at time $t$ following policy $\pi_\theta$ in $\mathcal{M}$, and let $\{Z_t\}_{t \geq 0}$ be the sequence of these costs. 
% \kamyar{Consider a filtration $F=\{\mathcal{F}_t\}_t$ such that $(s_0, a_0, \ldots, s_t)$ is $\mathcal{F}_t$-measurable.}
% Let $\mathcal{F}_t := \sigma (s_0, a_0, \ldots, s_t)$ denote a $\sigma$-algebra 
% generated by the states up until time $t$ and actions before time $t$. 
% \kamyar{You can say, $Z_t$ is $\mathcal{F}_t$-measurable}
% $Z_t$ is measurable with respect to $(\Omega, \mathcal{F}_t, P_\theta)$. 

%\kamyar{we need to work out this definition. }
\begin{definition}[Markov Coherent Risk Objective] 
% Given a sequence of random costs $\{Z_t\}_{t \geq 0}$, 
Given a sequence of random costs $\{C(s_t)\}_{t \geq 0}$, 
the Markov coherent risk objective is defined as
% \begin{align}
%     \rho_\theta(s_0) &= C(s_0) + \gamma \rho\Bigg(C(s_1) + ... \nonumber+\gamma \rho\bigg(C(s_{T-1}) \\
%     &\qquad \qquad + \gamma \rho\Big(C(s_{T}) + \ldots \Big)\bigg)\Bigg| s_0; \pi_\theta\Bigg).\label{def:mcr}
% \end{align} 
{%\small
\begin{equation}\label{def:mcr}
    \rho_\theta(s_0) = \Bigg(C(s_0) + \gamma \rho \bigg(C(s_1) + ... \nonumber+ \gamma \rho \Big(C(s_{t}) + \ldots \bigg); %s_0, %\Big| 
    \pi_\theta\Bigg). 
\end{equation} 
}%
\end{definition}
% \kamyar{what is $T$ here?}\leqi{Should we change $|s_0;\pi_\theta$ to $;\theta$?}
The MCR is a nested discounted sum of the risk functional $\rho$
applied at each timestep to the randomness of the cost $C(s_t)$ given the previous state $s_{t-1}$, which arises from the stochastic policy $\pi_\theta$ and transitions of the MDP.
% The MCR is a nested discounted sum of the risk functional $\rho$
% applied to the cost $C(s_t)$ at each time step\kamyar{it is not just $Z_t$} \leqi{Maybe we can just say $\rho(Z_t)$?}, where the randomness in each timestep arises from the policy $\pi_\theta$ and the transitions of the MDP. \kamyar{Probably we can define $Z_t = C(s_t)+\rho(Z_t+1)$. This way, $\rho_\theta(s_0)$ is $Z_0$. No need for having $T$. Also no need for the notation in Eq.~\ref{def:mcr}}
% It was first proposed by \citet{ruszczynski2010risk} and studied under the policy gradient setup by \citet{tamar2015policy}. 
% It is an upper bound on $\rho$ applied to the random returns due to the properties of coherent risk, and minimizing the MCR can be seen as minimizing the upper bound of the coherent risk. [discuss more]
We are interested in solving the following optimization problem, which seeks the parameter $\theta$ 
that minimizes the MCR \eqref{def:mcr},
\begin{equation}\label{eqn:mcrp}
    \min_\theta \rho_\theta(s_0)
\end{equation}
using policy gradient methods. 
% In Section \ref{sec:pgd}, we formally investigate the global optimality of policy gradient for the MCR. 
% We consider policy gradient methods for optimizing this objective, and investigate the global optimality of s. 
% While the policy gradient methods in \citet{tamar2015policy} aim to find locally optimal points, 
% in this work we formally investigate 
% the global optimality of policy gradient for this objective. 

\subsection{Value Function and Gradient}\label{SEC:VALUE}\label{sec:value}
The risk-sensitive \emph{value function} for an MDP $\mathcal{M}$ under the policy $\pi_\theta$ is defined as $V_\theta(s) = \rho_\theta(s)$. As shown in \citet{ruszczynski2010risk} and \citet{tamar2015policy}, the Bellman operator for 
$V_\theta(s)$ is
\begin{equation}\label{eqn:mcr_value}
    V_\theta(s) = C(s) + \gamma \max_{\xi \in \U(P_\theta)}\sum_{s'}P_\theta(s'|s)\xi(s')V_\theta(s'),
\end{equation}
where $\U$ is the risk envelope defined in Theorem \ref{thm:dual}  
% with $Z = V(s')$ \kamyar{instead of $Z = V(s')$, we should somehow imply that the prior works showed that the solution to Eq.~\ref{eqn:mcr_value} is equal to $Z$} \leqi{Perhaps changing $Z$ to $X$ in Section 3?} 
and $P_\theta(s'|s) := \sum_{a \in \mathcal{A}} \pi_\theta(a|s) P(s'|s,a)$.  
Under Assumption \ref{assum:risk_envelope}, the value function can be re-written as %a max-min Lagrangian problem 
\begin{align*}
       V_\theta(s) = C(s) + \gamma \max_{\xi}\min_{\lambda^P, \lambda^{\mathcal{E}}, \lambda^{\mathcal{I}}}L_{\theta, s}(\xi, \lambda^P, \lambda^\mathcal{E},  \lambda^\mathcal{I}), 
\end{align*}
where $L_{\theta, s}$ is the corresponding Lagrangian
%\scriptsize
{%\small
\begin{align}
    L_{\theta, s}(\xi, \lambda^P, \lambda^\mathcal{E},  \lambda^\mathcal{I}) &= \sum_{s'}P_\theta(s'|s)\xi(s')V_\theta(s') - \lambda^P\Big(\sum_{s'}P_\theta(s'|s)\xi(s') - 1\Big) \nonumber \\
    &\quad
    - \sum_{e \in \mathcal{E}}\lambda^\mathcal{E}(e)g_e(\xi, P_\theta) -  \sum_{i \in \mathcal{I}}\lambda^\mathcal{I}(i)f_i(\xi, P_\theta). 
\label{eqn:mcr_value_lagrang}
\end{align}
%--- uai
%\begin{align}
%    &L_{\theta, s}(\xi, \lambda^P, \lambda^\mathcal{E},  \lambda^\mathcal{I}) = \sum_{s'}P_\theta(s'|s)\xi(s')V_\theta(s') \nonumber \\
%    &\quad\quad\quad- \lambda^P\Big(\sum_{s'}P_\theta(s'|s)\xi(s') - 1\Big) 
%    - \sum_{e \in \mathcal{E}}\lambda^\mathcal{E}(e)g_e(\xi, P_\theta) \nonumber \\
%    &\quad\quad\quad-  \sum_{i \in \mathcal{I}}\lambda^\mathcal{I}(i)f_i(\xi, P_\theta). 
%\label{eqn:mcr_value_lagrang}
%\end{align}
}%
% and where $\xi \in \R^{|\Ss|}$, $\lambda^P \in \R$, 
% $\lambda^\mathcal{E} \in \R^{|\mathcal{E}|}$, and 
% $\lambda^\mathcal{I} \in \R_{\geq 0}^{|\mathcal{I}|}$.
% For any vector $\lambda \in \mathbb{R}^d$ and $i \in \{1, \ldots, d\}$, 
% $\lambda(i)$ is the $i$-th coordinate of $\lambda$. 
% For $\xi \in \R^{|\Ss|}$, we use $\xi(s)$ to refer to 
% the entry of $\xi$ corresponding to the state $s$. 
For each state $s \in \Ss$, let 
$(\xi_{\theta, s}, \lambda_{\theta, s}^{P}, \lambda_{\theta, s}^{\mathcal{E}}, \lambda_{\theta, s}^{\mathcal{I}})$ 
denote a saddle point 
of \eqref{eqn:mcr_value_lagrang},
which is guaranteed to exist %\kamyar{you many not need to say "under  Assumption " since the equation holds under that case, and you stated it.} 
under Assumption \ref{assum:risk_envelope}. 
The gradient of the value 
$\nabla_\theta V_\theta(s_0) = \nabla_\theta \rho_\theta(s_0)$ 
is calculated using the saddle points 
$(\xi_{\theta, s}, \lambda_{\theta, s}^{P}, \lambda_{\theta, s}^{\mathcal{E}},  \lambda_{\theta, s}^{\mathcal{I}})$, and was derived in \citet{tamar2015policy}.
We restate this theorem below with proof in Appendix \ref{sec:appendix_value}.  

\begin{theorem}%[Gradient of MCRP \cite{tamar2015policy}]
\cite[Theorem 5.2]{tamar2015policy}
\label{thm:mcrp_gradient}
Under Assumption \ref{assum:risk_envelope} 
and assuming the likelihood ratio 
$\nabla_\theta \log \pi_\theta(a|s)$ 
is well-defined and bounded for all $(s, a)$,  
\begin{equation*}\label{eqn:mcr_gradient}
    \nabla_\theta V_\theta(s) = \E_{\xi_{\theta, s}}\left[ \sum_{t=0}^\infty \gamma^t \nabla_\theta \log \pi_\theta(a_t|s_t) h_\theta(s_t, a_t) | s_0 = s\right],
\end{equation*}
where $\E_{\xi_{\theta, s}}$ refers to expectation 
with respect to trajectories generated
by a Markov chain with transition probabilities 
$P_\theta(\cdot|s)\xi_{\theta, s}(\cdot)$, and 
%\scriptsize
\begin{align}\label{eqn:mcr_h_theta}
    h_\theta(s, a) &= \gamma\sum_{s'}P(s'|s,a)\bigg[\xi_{\theta, s}(s')\Big(V_\theta(s')   - \lambda_{\theta, s}^{{P}}\Big) \nonumber \\
    &\quad- \sum_i \lambda_{\theta, s}^{\mathcal{I}}(i)\frac{df_i(\xi_{\theta, s}, P_\theta)}{dP_\theta} - \sum_e \lambda_{\theta, s}^{\mathcal{E}}(e)\frac{dg_e(\xi_{\theta, s}, P_\theta)}{dP_\theta} \bigg]. 
\end{align}
\end{theorem}

\paragraph{Notation}
% \kamyar{define these in the preliminaries. We start section 5, and say, hey we are going to show you so many great results. Then we say, wait, shoot, I forgot to tell you the notations. This should go to the preliminaries. It breaks the threat of thoughts in short papers.}
We now introduce some notation key 
to our analysis of policy gradient in the following sections. 
First, define the infinite discounted 
state visitation distribution 
under a policy $\pi_\theta$ 
and the starting state $s_0$ to be
$d_{\pi_\theta} (s) := (1-\gamma)\sum_{t=0}^\infty \gamma^t P_{\pi_\theta}(s_t=s|s_0)$,
where $P_{\pi_\theta}(s_t=s|s_0)$ 
is the probability that $s_t = s$ 
after executing $\pi_\theta$ 
starting from state $s_0$. 
We adopt the shorthand for these terms 
as $d_\theta$ and $P_\theta$, respectively, 
and $d_*$ and $P_*$ for the optimal policy $\pi_*$. 
Without loss of generality, 
we assume that that the starting state is deterministic
and thus remove $s_0$ from our notation to avoid clutter. 
We next define the MDP transitions 
induced by $\pi_\theta$ 
and reweighted by $\{\xi_s\}_{s \in \Ss}$ 
to be $P_{\theta}^\xi$. 
Under this reweighting, the transitions are given 
by $P_\theta^\xi(s'|s) := P_\theta(s'|s)\xi_s(s')$. 
The state visitation distribution 
under these reweighted transitions 
is $d^{\xi}_\theta (s) := (1-\gamma)\sum_{t=0}^\infty \gamma^t P_{\theta}^\xi(s_t=s|s_0)$.

\section{Global Optimality}\label{SEC:GRADIENT_DOMINATION}\label{sec:gradient_domination}
% \begin{figure}
%     \centering 
%     \includegraphics[width=1\linewidth]{images/bandit_optimization.png}
%     \caption{The optimization landscape for MCR using CVaR at different $\alpha$ levels in our illustrative 2-armed bandit problem. } 
%     %\label{fig:bandit}
%     \label{fig:bandit_landscape} 
% \end{figure}

In this section, we investigate 
the global optimality and convergence 
of PG for the MCR objective~\eqref{eqn:mcrp}
and consider policies directly parameterized by 
$\theta \in \R^{|\Ss|\times|\A|}$, i.e. $\pi_\theta(a|s) = \theta_{s,a}$. 
We begin with an illustrative example demonstrating that risk-sensitive MCR exhibits local minima for even simple problems, which suggests that unlike the expected return, it is not gradient dominated. 
% to demonstrate the sensitivity 
% of PG algorithms for MCR
% to the initial policy.
% As seen in Figure \ref{fig:bandit_landscape},
% the MCR exhibits local minima,
% which already suggests 
% that unlike expected return, 
% it is not gradient dominated.
We formalize this result in Lemma \ref{LEM:GRADIENT_DOMINATION},
which states that the value optimality gap 
$V_\theta - V_*$ under the MCR 
is upper bounded by the magnitude 
of the gradient plus residual terms. 
As a result, stationary points of the MCR objective 
are not guaranteed to be global optima.
The residuals are a consequence of the fact that, 
unlike in the expected return, 
the MCR Bellman operator~\eqref{eqn:mcr_value} is nonlinear in the policy $\pi_\theta$ due to the maximization over $\U$, 
which is dependent on $\pi_\theta$. 
We show that the residuals vanish under the expected return, which is a Markov coherent risk objective, demonstrating that Lemma \ref{LEM:GRADIENT_DOMINATION} is a generalization of the gradient domination lemma previously developed in \citet{agarwal2019theory} for expected return.
% Moreover, we prove that the residuals 
% are not merely an artifact of our analysis
% but represent an actual obstacle to optimization,
% by introducing a lower bound (Theorem~\ref{thm:lower_bound})
Moreover, we prove that there exist problem instances for which the gradient domination inequality holds with equality, demonstrating the tightness of our result.  
Finally, using the gradient domination lemma, we establish the convergence rate of projected gradient descent for the MCR objective.

\begin{figure}
    \centering
    \begin{subfigure}[b]{0.4\textwidth}
    \centering
    \begin{tikzpicture}[scale=0.7, transform shape]
        \node[state, ] (q1) {$ $};
        \node[state, right of=q1, xshift=0.5cm] (q3) {$C=0.5$};
        \node[state, above of=q3] (q2) {$C=0$};
        \node[state, below of=q3] (q4) {$C=1$};
        \node[state, accepting, right of=q3] (q5) {$ $};
        \draw 
                (q1) edge[above] node{$.2\theta$} (q2)
                (q1) edge[above] node{$1-\theta$} (q3)
                (q1) edge[below] node{$.8\theta$} (q4)
                (q2) edge[above] node{1} (q5)
                (q3) edge[above] node{1} (q5)
                (q4) edge[above] node{1} (q5)
                (q5) edge[loop right] node{} (q5);
    \end{tikzpicture}
    \caption{} 
    \label{fig:bandit_problem}
    \end{subfigure}
    \hfill
    \begin{subfigure}[b]{0.55\textwidth}
    \centering 
    \includegraphics[width=1\linewidth]{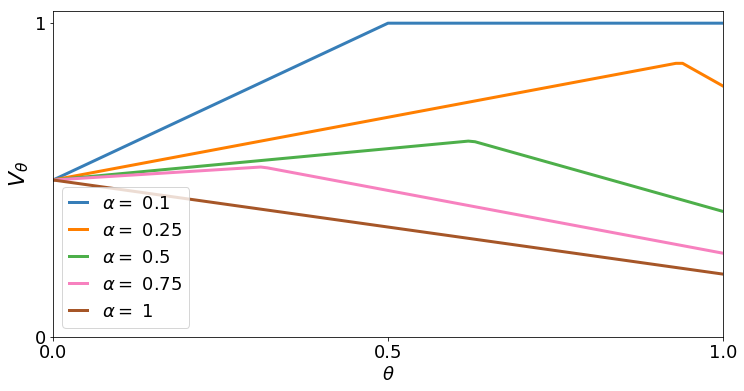}
    \caption{} 
    %\label{fig:bandit}
    \label{fig:bandit_landscape} 
    \end{subfigure}
    \caption{(a) The MDP for the bandit problem given in Section~\ref{sec:motivating}. (b) The optimization landscape for MCR using CVaR at different $\alpha$ levels in this MDP.}
    \label{fig:lower_bound_mdp}
\end{figure}
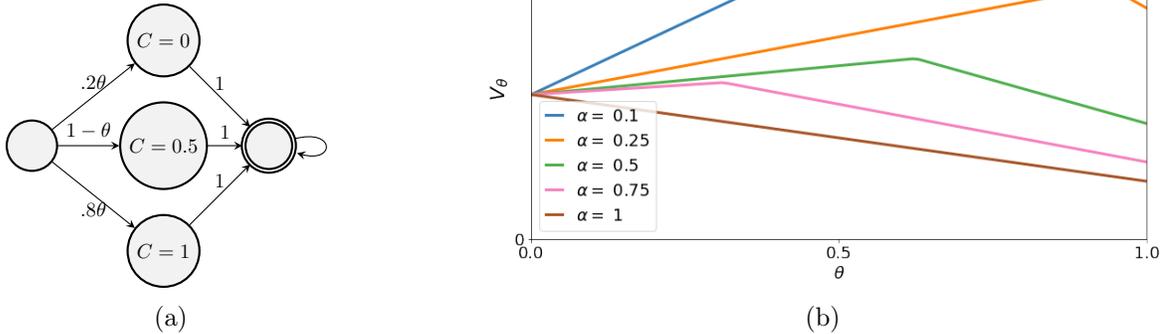

\subsection{Motivating Example}
\label{sec:motivating}
Consider a 2-armed bandit problem
where the cost of Arm 1 is a Bernoulli random variable 
which assigns probability $0.8$ 
to receiving cost $1$,   
and the cost of Arm 2 is always $0.5$.
Suppose we work with a directly parameterized policy, 
where the probability of playing arm 1 
is given by $\theta$ 
and the probability of playing arm 2 
is given by $1-\theta$. 
This problem can equivalently be written as a five-state MDP, with transitions given in Figure \ref{fig:bandit_problem}.

We use the MCR objective 
with the $\cvara$ risk functional, and in Figure \ref{fig:bandit_landscape},
% deterministically 
% the constant .
%in Figure \ref{fig:bandit_problem}. 
% Arm 1 gives cost 0 with 2 or 1 with some nonzero probability\leqi{Maybe we should be explicit about what the probability is in this plot.}, and arm 2 always gives an intermediate cost of 0.5. 
% 
% Though we frame this motivating example as a bandit problem, 
% which end in a self-absorbing terminal state\kamyar{transitions? and reward? You may want to add a drawing. Or you may just start with the MDP and leave the bandit description aside.}.
plot the optimization landscape for this problem, 
i.e., the MCR objective over the range 
of possible policies $\theta$,
at different risk sensitivity levels $\alpha$. 
It is clear that, for a range of $\alpha < 1$ levels, 
two stationary points exist 
at the boundary $\theta=0$ and $\theta = 1$.
% which point PG converges to depends on the initialization. 
Which one PG converges to depends 
on the initial value of $\theta$.
% For a range of initial policies, gradient descent will converge to a suboptimal stationary point. 
Only when CVaR  $\alpha = 1$,
which corresponds to the expected return, 
will the stationary point 
be unique and globally optimal. 
This matches previous results \citep{agarwal2019theory} demonstrating that the expected return
% which is linear in the policy $\theta$,
is gradient dominated.  

In addition, this counterexample exposes another source 
of poor convergence and suboptimality in MCR objectives. 
Because certain coherent risk functionals, 
such as $\cvara$, consider a worst-case 
reweighted fraction of the cost distribution, 
they are more 
% prone to suffering from
vulnerable to vanishing gradients 
than the expected return.
% cost objective is. 
This problem is exacerbated 
for higher degrees of risk sensitivity, 
e.g., lower $\alpha$ for CVaR. 
As can be seen from the $\alpha = 0.1$ line 
in Figure \ref{fig:bandit_landscape}, 
the gradient is zero for a significant fraction 
of random initializations of $\theta$,
while expected value ($\alpha = 1$) 
has no such problem.  

\subsection{Gradient Domination}
We formalize the intuition of our motivating example 
using a property called gradient domination. 
Any stationary point 
of a gradient dominated function 
is also globally optimal.  
Even if a gradient dominated function 
is highly non-convex, 
any optimization algorithm 
that converges to a stationary point 
will also converge to a global optimum. 
The expected 
return
% value objective 
has previously been shown
to be gradient dominated. 
We follow the notion of gradient domination 
introduced by \citet{bhandari2019global} 
and formally articulated in Definition \ref{def:gradient_domination}. 

\begin{definition}\label{def:gradient_domination}
For $\theta \in \R^d$, a function $f$ 
is $(c, \mu)$- gradient dominated over $\Theta$ 
if there exists constants $c > 0$ and $\mu \geq 0$ 
such that for all $\theta \in \Theta$, 
\begin{equation}\label{eq:gd}
    \min_{\theta' \in \Theta}f(\theta') \geq f(\theta) + \min_{\theta' \in \Theta}\Big[c\langle\nabla f(\theta), \theta' -\theta \rangle + \frac{\mu}{2}\Vert \theta - \theta'\Vert_2^2 \Big]. 
\end{equation}
\end{definition}

%Eq. \eqref{eq:gd} 
Definition~\ref{def:gradient_domination} 
has two important implications. 
First, when $\langle \nabla f(\theta), \theta' - \theta \rangle \geq 0$ for all $\theta'$,
which occurs when $\theta$ is a stationary point, 
we have $\min_{\theta' \in \Theta} f(\theta') \geq f(\theta)$,
implying that $\theta$ is globally optimal.
Second, for any $\theta$, 
the optimality gap $f(\theta) - \min_{\theta' \in \Theta}f(\theta')$ 
can be upper bounded by the minimization problem 
on the right hand side, 
which is a measure of how far 
$\theta$ is from stationary.
We demonstrate through 
Lemma \ref{lem:gradient_domination} 
that for 
% Markov coherent risk, 
MCR
although the first statement does not hold, 
we can still expressively bound the optimality gap 
with a similar optimization problem. 

We first define the two residuals 
presented in Lemma \ref{lem:gradient_domination}, 
which arise from two sources of nonlinearity 
in the Markov coherent risk value function: 
(i) the dependence of the saddle point solutions
$\xi_{\theta,s}$ and $\lambda_{\theta,s}$ on $\pi_\theta$;
and (ii) in the constraints $f_i$ and $g_e$ which, depending on the coherent risk measure, 
may be highly nonlinear in $\pi_\theta$.
The first source of nonlinearity 
contributes to our analysis is the difference 
between the primal and dual solutions 
of $\pi_\theta$ and $\pi_*$, 
which we upper bound by $\epsilon_{L}$ 
in Assumption \ref{assum:eps_l}. 
The second source of nonlinearity 
contributes a residual which we upper bound
by $\epsilon_{U}$ in Assumption \ref{assum:eps_u}. 
The residual $\epsilon_{U}$ is related to 
the difference between the constraint terms 
in the Lagrangian~\eqref{eqn:mcr_value_lagrang}
and their first-order approximation 
under $\pi_\theta$ in the gradient update, 
and can be interpreted as a measure 
of how nonlinear the constraints are 
in the policy $\pi_\theta$. 

\begin{assumption}[Policy Residual]\label{assum:eps_l}  
Given a policy $\theta$ and any other policy $\overline\pi$, 
{%\small 
\begin{align}\label{eqn:policy_res}
    \epsilon_L(\theta, \overline\pi) := \gamma\sum_{s}d_\theta^{\xi_\theta}(s)\Big|\Big(\lambda_{\theta,s}^{P} - \lambda_{\overline\pi,s}^{P}\Big)\sum_{s'}\Big(P_\theta - P_{\overline\pi}\Big)(s'|s)\xi_{\theta,s}(s')\Big|
\end{align}
}%
Suppose that for all $\theta$ and $\overline\pi$, $\epsilon_L(\theta, \overline\pi) \leq \epsilon_L$. % some abuse of notation?
% \leqi{Sould we change $;$ to $,$? I think people may think $\pi$ is parametrized by $\theta$?}
\end{assumption}

\begin{assumption}[Constraint Residual]\label{assum:eps_u}
Given a policy $\theta$ and any policy $\overline\pi$, 
{%\small
\begin{align}
    \epsilon_U(\theta, \overline\pi) := &\gamma \sum_s d_{\theta}^{\xi_\theta}(s) \Big|\sum_i \lambda_{\overline\pi,s}(i) \Delta f_i + \sum_e\lambda_{\overline\pi,s}(e) \Delta g_e \nonumber \\
    &+ \Big(\pi_\theta - \overline\pi\Big)^\top \Big(\sum_i \lambda_{\theta, s}(i)\nabla_\theta f_i(\xi_\theta, P_\theta) + \sum_e\lambda_{\theta, s}(e)\nabla_\theta g_e(\xi_\theta, P_\theta) \Big) \Big| \label{eqn:constraint_res}
\end{align}
}%
where $\Delta f_i = f_i(\xi_\theta, P_{\overline\pi}) - f_i(\xi_\theta, P_\theta)$
and $\Delta g_e$ is defined similarly.
Suppose that for all $\theta$ and $\overline\pi$, $\epsilon_U(\theta, \overline\pi) \leq \epsilon_U$. 
\end{assumption}

% \begin{assumption}[Residual Error]
% Suppose that the following holds for all $\theta$: 
% \begin{align*}
%  \E_{\Big[\Big|\Big(\lambda_{\theta, s}^{*, P} - \lambda_{*, s}^{*, P}\Big)
%  \sum_{s'}P_\theta(s'|s)\Big(\xi_{*,s}(s') - \xi_{\theta, s}(s')\Big)\Big|\Big] 
%  \leq \epsilon_{\text{res}},
% \end{align*}
% \end{assumption}

% \begin{assumption}[Constraint Residual]\label{assum:error_constraint} 
% Suppose that the following holds for all $\theta$: 
%     \begin{align*}
%     &\mathbb{E}_{s \sim d_{\xi_*}}\bigg[\Big|\sum_{i \in \mathcal{I}} \lambda_{*, s}^{*, \mathcal{I}}(i)\Big(f_i(\xi_{\theta,  s}, P_*) - f_i(\xi_{\theta, s}, P_\theta)\Big) + \sum_{e \in \mathcal{E}} \lambda_{*, s}^{*, \mathcal{E}}(e)\Big(g_e(\xi_{\theta,  s}, P_*) - g_e(\xi_{\theta,  s}, P_t)\Big)\\
%     &+ \sum_{a} \Delta \pi_\theta(a|s) \sum_{s'} P(s'|s,a) \Big(\sum_{i} \lambda_{\theta,  s}^{*, \mathcal{I}}(i)\frac{df_{i}(\xi_{\theta, s}, P_\theta)}{dp(s')} + \sum_{e} \lambda_{\theta, s}^{*, \mathcal{E}}(e)\frac{dg_{e}(\xi_{\theta, s}, P_\theta)}{dp(s')}\Big)\Big| \bigg] \leq
%     \epsilon_{\text{cons}},
%     \end{align*}
%     \normalsize
% where $\Delta \pi_\theta(a|s) = \pi_\theta(a|s) - \pi_*(a|s)$.
% \end{assumption}

% \leqi{Perhaps we could add a paragraph about performance difference lemma here and refer the readers to the appendix?}
We first derive a performance difference lemma for the MCR (Appendix Lemma \ref{lem:performance_difference}), which upper bounds the optimality gap using an MCR advantage function. We then use it to establish the gradient domination bound in Lemma \ref{LEM:GRADIENT_DOMINATION}. 

\begin{lemma}[Gradient Domination]\label{LEM:GRADIENT_DOMINATION}\label{lem:gradient_domination} %\label{lem:gradient_domination} 
The optimality gap $V_\theta(s_0) - V_*(s_0)$ is upper bounded as 
%\aud{Actually here, we decided not to use Assumption 2 and 3 right?} 
% Suppose that Assumptions~\ref{assum:eps_l} 
% and~\ref{assum:eps_u} hold for a policy $\pi_\theta$\leqi{Assumption 2 seems to require a $\theta$ and a $\pi$?}, then 
% \begin{align*}
%     V_\theta(s_0) - V_*(s_0) &\leq \Big\Vert \frac{d}{d^\theta}\Big\Vert_\infty \max_{\overline{\pi} \in \Delta(\A)^{\Ss}} (\overline{\pi} - \pi_\theta)^\top \nabla_\theta V_\theta(s_0) \\
%     &\quad+ \gamma \Big\Vert \frac{d_{\xi_\theta}}{d_{\xi_*}} \Big\Vert_\infty \Big(\epsilon_{\text{res}} + \epsilon_{\text{cons}}\Big). 
% \end{align*}
\begin{align*}
    V_\theta(s_0) - V_*(s_0) &\leq \Big\Vert \frac{d_*^{\xi_\theta}}{d_\theta^{\xi_\theta}}\Big\Vert_\infty \Big(\max_{\overline{\pi} \in \Delta(\A)^{\Ss}} (\overline{\pi} - \pi_\theta)^\top \nabla_\theta V_\theta(s_0) + \epsilon_{L}(\theta, \pi_*) + \epsilon_{U}(\theta, \pi_*)\Big). 
\end{align*}
% where $\epsilon_L(\pi_*, \theta)$ and $\epsilon_U(\pi_*, \theta)$ are defined in Definitions \ref{assum:eps_l} and \ref{assum:eps_u} respectively. 
\end{lemma}
Lemma \ref{lem:gradient_domination} demonstrates 
that it is not 
% simply 
sufficient 
for the gradient of the value to be small 
for the policy to be nearly optimal.
% in terms of its value.
As with previous results for expected return, 
the $\xi_\theta$-reweighted 
state visitation distribution $d_\theta^{\xi_\theta}$ 
must adequately cover the support of 
the reweighted state visitation distribution $d_*^{\xi_\theta}$ under the optimal policy. 
Whenever the fraction 
$\Vert d_* / d_\theta \Vert_\infty$,
which is present in the gradient domination bound
for the expected return \citep{agarwal2019theory}
is finite,
the term $\Vert d_*^{\xi_\theta} / d_\theta^{\xi_\theta} \Vert_\infty$
is also finite by definition 
(the converse is not necessarily true).

Further, even if $\theta$ is a stationary point, i.e., $(\overline{\pi} - \pi_\theta)^\top \nabla_\theta V_\theta(s_0) \leq 0$ for all $\overline{\pi}$, 
the residuals $\epsilon_{L}(\theta, \pi_*)$ and $\epsilon_{U}(\theta, \pi_*)$ 
can be nonzero and global optimality is not guaranteed. 
% The residuals serve measure of the distance between
The residuals capture the distance between 
$\pi_\theta$ and an optimal policy $\pi_*$ 
in terms of the Lagrangian primal and dual variables 
that are used to calculate the objective function. 
In general, an optimal policy 
for a given MDP is not known a priori, 
but we can upper bound the optimality gap 
at a stationary point $\theta$ using $\max_{\overline{\pi}}\Big[\epsilon_{L}(\theta, \overline{\pi}) + \epsilon_{U}(\theta, \overline{\pi})\Big] \leq \epsilon_L + \epsilon_U$ from Assumptions \ref{assum:eps_l} and \ref{assum:eps_u}. 
The upper bound is problem-dependent 
and reflects the magnitude of risk aversion of the chosen coherent risk functional, 
and degree of nonlinearity in the policy arising from the risk envelope $\U$ in the Bellman operator~\eqref{eqn:mcr_value}. 
% This nonlinearity is a consequence of the maximization over the risk envelope $\U$, which is dependent on the policy $\pi_\theta$. 
% As previously mentioned, 
The residual $\epsilon_L$ has a larger upper bound 
when $\xi$ is permitted to be large, e.g., 
smaller fractions of the cost distribution 
are given a high weight. 
The residual $\epsilon_U$ is a measure of the error 
of a first-order approximation 
of the constraints $f_i$ and $g_e$, 
and is larger when the constraints 
are highly nonlinear in $\pi_\theta$. 
Roughly, the magnitude of $\epsilon_L$ is small 
when the saddle points of $\pi_\theta$ 
are close to the saddle points 
of the optimal policy $\pi_*$, 
and the residual $\epsilon_U$ is small 
if the constraints are additionally 
linear in $\pi_\theta$.  

The expected value objective is equivalently formulated as the MCR objective for $\cvar$ with $\alpha = 1$, and its risk envelope is a singleton of $\U_{\cvar_1}(P_\theta) = \{\xi:\xi = 1\}$. 
As a result, $\xi_{\theta,s} = 1$ for all $\theta, s$ and the expected value Bellman operator is linear in the policy. 
Then by definition, the residuals $\epsilon_{U} = \epsilon_{L} = 0$ for expected value, and 
Lemma~\ref{lem:gradient_domination} 
is a generalization of the gradient domination lemma 
for expected return, i.e., 
(Lemma 4.1 in \citet{agarwal2019theory}).
For $\text{CVaR}_\alpha$, 
we have $\epsilon_{U} = 0$
as the inequality constraints $0 \leq \xi \leq \frac{1}{\alpha}$ are independent of $\pi_\theta$, 
and $\epsilon_L \propto \frac{1}{\alpha}$. 
As a result, the upper bound on the optimality gap 
is larger when $\alpha$ is smaller 
and the objective is more risk sensitive. 

Moreover, Theorem \ref{thm:lower_bound} guarantees
that the bound in Lemma \ref{lem:gradient_domination} 
is in fact tight. We provide a proof sketch, with formal argument in Appendix \ref{sec:appendix_gradient_domination}.

\begin{theorem}%[Lower Bound of Gradient Domination]
\label{THM:LOWER_BOUND}\label{thm:lower_bound}
There exists an MDP $\mathcal{M}$ 
and $\pi_\theta$ such that 
\begin{align*}
    V_\theta(s_0) - V_*(s_0) &= \Big\Vert \frac{d_*^{\xi_\theta}}{d_\theta^{\xi_\theta}}\Big\Vert_\infty \Big(\max_{\overline{\pi} \in \Delta(\A)^{\Ss}} (\overline{\pi} - \pi_\theta)^\top \nabla_\theta V_\theta(s_0) + \epsilon_{L}(\theta, \pi_*) + \epsilon_{U}(\theta, \pi_*)\Big).
\end{align*}
\end{theorem}
\begin{proof}[Proof Sketch.]
We give an informal argument on
the existence of the lower bound 
and the necessity of the residual terms. 
Take the CVaR at level $\alpha$ objective for example, 
and consider a 2-action MDP, 
where one action always returns a set high cost, 
and the other action returns a set low cost most of the time and the high cost the rest. 
It is easy to see that the optimal policy will always choose the latter action. 
However, there is a large fraction of policies in this MDP where, 
even if the policy chooses both actions with reasonable probabilities,
the highest $\alpha$-quantile 
of the return distribution 
may include only the high cost. 
Thus, even though the policy may be arbitrarily suboptimal, it will have a gradient of 0. 
The suboptimality of the policy can be quantified exactly using the residuals terms $\epsilon_{L}(\theta, \pi_*)$ and $\epsilon_{U}(\theta, \pi_*)$. 
\end{proof}

\begin{remark}
Previous studies on the global optimality 
of PG for the expected return 
have taken into account 
two sources of suboptimality:
vanishing gradients \cite{agarwal2019theory} 
and closure under policy improvement, 
which \citet{bhandari2019global} used 
as a condition for why some problems 
with constrained policy classes 
are gradient dominated 
and others are not. 
Our results demonstrate that 
neither of these is sufficient 
to explain suboptimality 
in the case of coherent risk objectives. 
For the latter, it is clear 
in our motivating example 
that the policy class $\theta \in [0, 1]$ 
is closed under policy improvement. 
\end{remark}

\subsection{Convergence Rates}\label{sec:pgd}
We consider the convergence rate 
of projected gradient descent (PGD) 
on the MCR under direct policy parameterization,
% which takes updates of the form 
where updates take the following for:
\begin{equation}\label{eqn:pgd}
    \theta \leftarrow \text{Proj}_{\Delta(\A)^{\Ss}}\Big(\theta - \eta \nabla_\theta V_{\theta}(s_0)\Big),
\end{equation}
where $\text{Proj}_{\Delta(\A)^{\Ss}}$ is a projection 
on the probability simplex $\Delta(\A)^{\Ss}$ 
in the Euclidean norm 
and $\eta$ is the learning rate. 
Using the gradient domination lemma, 
we can provide the following 
iteration complexity bound on PGD,
with proof in Appendix \ref{proof:pgd}:  
\begin{proposition}\label{thm:pgd}\label{THM:PGD}
If $V_\theta$ is $\beta$-smooth in $\theta$, 
then the PGD algorithm in (\ref{eqn:pgd}) 
on $V_\theta(s_0)$ with step size 
$\eta = \frac{1}{\beta}$ satisfies
{%\small
\begin{align*}
    \min_t V_{\theta_t}(s_0) - V_*(s_0) &\leq D\sqrt{\frac{32|\Ss|\beta}{(1-\gamma)T}} + \epsilon_L + \epsilon_U,
\end{align*}
}
where $D = \max_t\Vert d_*^{\xi_{\theta_t}} / d_{\theta_t}^{\xi_{\theta_t}}\Vert_\infty$.
\end{proposition}
% \kamyar{does not seem to be necessary to be stated. Instead we can talk more about the results itself.}
% \aud{I originally took this sentence from Alekh's paper, since I was also wondering at the time why the result wasn't given for the last value $V_{\theta_T}$. 
% We can take it out if you don't think it's necessary. }
This guarantee is provided 
for the best policy over $T$ rounds,
as is standard in the nonconvex optimization literature. 
This lemma demonstrates that gradient descent converges in $O(1/\sqrt{T})$ iterations, but is upper bounded by the same residuals present in the gradient domination lemma (Lemma \ref{LEM:GRADIENT_DOMINATION}). Even if $T \rightarrow \infty$ and the first term vanishes, it is possible for the policy to be suboptimal because the MCR objective is not gradient dominated. This phenomenon is evident in the motivating example of Section \ref{sec:motivating} and the optimization landscape in Figure \ref{fig:bandit_landscape}. 
However, if $\epsilon_U$ and $\epsilon_L$ are additionally small, which occurs when $\theta$ is near $\pi_*$ or the MCR Bellman operator is (close to) linear in the policy, Proposition \ref{THM:PGD} demonstrates that $\theta$ is near-optimal as well. 
Finally, as shown previously, when the Bellman operator is linear in the policy (as is the case for expected value), the residuals vanish and we recover the convergence guarantees for PGD from \citet{agarwal2019theory} for expected value. 

% The optimality gap shrinks as $O(1/\sqrt{T})$,
% as long as $\epsilon_{L}$
% and $\epsilon_{U}$ are relatively small.

Coherent risk functionals are not, 
in general, $\beta$-smooth.
CVaR, for example, is considered a non-smooth 
optimization problem \citep{alexander2006minimizing}, 
and in optimizing this objective,
the authors in \citet{tamar2014optimizing} 
make the assumptions that $\frac{\partial \rho_\theta(Z)}{\partial \theta}$ 
and $\frac{\partial \lambda_\theta(\omega)}{\partial \theta}$ 
exist and are bounded. 
Under similar assumptions, 
we can guarantee $\beta$-smoothness of $V_\theta$ and 
applicability of the bound in Proposition \ref{thm:pgd} 
(proof in Appendix \ref{lem:smoothness}):
\begin{corollary}\label{cor:smoothness}
If the gradients of the primal and dual solutions 
of the MCR value function (\ref{eqn:mcr_value}) exist 
and are bounded everywhere, i.e.,
%\small
$$\Big\Vert \frac{d\xi_\theta}{d\theta}\Big\Vert_\infty \leq c\quad \text{and}\quad \Big\Vert \frac{d\lambda_\theta}{d\theta}\Big\Vert_\infty \leq c'$$ 
%\normalsize
and the second derivative of the constraints exist 
and are bounded, i.e., for all 
$e \in \mathcal{E}$ and $i \in \mathcal{I}$
%\small
$$\Big\Vert \frac{d^2g_e(\xi, P)}{dP^2}  \Big\Vert \leq M' \quad\text{and}\quad \Big\Vert \frac{d^2f_i(\xi, P)}{dP^2}  \Big\Vert \leq M'$$
\normalsize
then there exists $\beta$ such that 
the value $V_\theta$ is $\beta$-smooth 
for starting state $s_0$, 
and Proposition \ref{thm:pgd} holds for $V_\theta$. 
\end{corollary}

\section{Gradient Estimation}\label{SEC:GRADIENT_ESTIMATION} \label{sec:gradient_estimation} 

In our analysis of the global optimality of PG for MCR, 
we assumed access to the gradient~\eqref{eqn:mcr_gradient}.
In practice however, the gradient is difficult to estimate 
because it is calculated with respect 
to an expectation over reweighted transitions $P_\theta^{\xi_\theta}$. 
\citet{tamar2015policy} propose 
an actor-critic method for optimizing the MCR,
but require the ability to sample from an MDP with reweighted transitions $P_\theta^{\xi_\theta}$,
which is often unavailable in practice. 
% \kamyar{a privileged access} access to a privileged access to a simulator 
% in order to estimate the gradient.
% Specifically, they require the ability 
% to sample from $P_\theta^{\xi_\theta}$,
% which may not be available in practice. 
% Without access to such a simulator, 
Absent such a resource, one naive method 
for optimizing the MCR 
is to draw samples from $d_\theta$ 
and multiply each sample 
by its importance weight 
$\prod_{j=1}^{t-1}\xi_{\theta, s_{j-1}}(s_{j})$, 
shown in Algorithm \ref{algo:is_gradient}. 

\begin{algorithm}
\SetAlgoLined
\KwIn{}
\For{episode $k=1...K$}{
Generate $N$ trajectories $\tau_k = \{s_0^{(n)}, a_0^{(n)}, c_0^{(n)}, ..., s_T^{(n)}, a_T^{(n)}, c_T^{(n)}\}_{n = 1}^N$\; 
Solve the maximization problem in \eqref{eqn:mcr_value_lagrang} to obtain $\widehat\xi_k, \widehat\lambda_k^P, \widehat\lambda_k^{\mathcal E}, \widehat\lambda_k^{\mathcal I}$ for all $(s, s')$\;
Compute $\widehat h_k(s,a)$ in \eqref{eqn:mcr_h_theta}\;
Compute $\nabla_\theta V_\theta = \sum_{n=1}^N \sum_{t=0}^T \gamma^t \Big(\prod_{j=1}^{t-1}\widehat\xi_{k, s_{j-1}}(s_{j})\Big)\widehat{h}_k(s_t,a_t)$\;
Update $\theta_{k+1} \leftarrow \theta_k - \eta \nabla_\theta V_\theta$\;
}
\caption{Importance Sampling Gradient Estimation}\label{algo:is_gradient}
\end{algorithm}
However, as demonstrated by \citet{liu2018breaking}, importance sampling suffers from variance 
that is exponential in the time horizon. 
This especially concerning 
for smaller CVaR levels $\alpha$, 
as the maximum per-step importance weight 
is $\frac{1}{\alpha}$, 
which can cause the magnitude of the gradient 
to become intractably large in longer horizons. 
Instead, we propose a practical method 
for optimizing the MCR that leverages 
recent advances in off-policy 
evaluation and optimization.
Moreover, we empirically demonstrate 
its improved stability and efficacy.

\subsection{State Distribution Reweighting}
Note that, using the stationary state distribution, 
we can write the gradient in \eqref{thm:mcrp_gradient} 
as 
\begin{align}
    \nabla_\theta V_\theta &= \sum_{s}d_{\theta}^{\xi_\theta}(s) \sum_a \pi_\theta(a|s) \nabla_\theta \log \pi_\theta(a|s) h_\theta(s,a) \nonumber \\
    &= \sum_{s : d_\theta(s) > 0}d_{\theta}(s)w_{\xi_\theta}(s)\sum_a \pi_\theta(a|s) \nabla_\theta \log \pi_\theta(a|s) h_\theta(s,a), \label{eqn:ssd_grad} 
\end{align}
where $w_{\xi_\theta}(s) := d_\theta^{\xi_\theta}(s) / d_\theta(s)$, 
which is well-defined because 
$d_\theta^{\xi_\theta}$ is always 0 
whenever $d_\theta$ is 0 by definition. 
If we can estimate the ratio $w_{\xi_\theta}$, 
then we can calculate the gradient 
by sampling from the original 
stationary state distribution $d_\theta$. 
We now state a theorem that can be used 
to derive the state distribution 
correction factor $w_{\xi_\theta}$, 
inspired by Theorem 4 of \citet{liu2018breaking}. 

\begin{theorem}\label{thm:correction}\label{THM:CORRECTION}
For any $\gamma \in (0, 1)$ 
and reweighting $\{\xi_s\}_{s \in \Ss}$ where $\xi_s:\Ss\rightarrow\R$, define  
\begin{align}
    L(w,f)&:= \gamma \E_{s,s' \sim d_\theta}\Big[\Delta(w; s,a,s')f(s') \Big] + (1-\gamma)\E_{s \sim d_0}\Big[\Big(1-w(s)\Big)f(s)\Big], \label{eqn:correction}
\end{align}
where $\Delta(w; s,a,s') := w(s)\xi_{s}(s') - w(s')$. 
Then $w(s)$ equals $w_{\xi}(s):= d^{\xi}_\theta(s)/d_\theta(s)$ 
if and only if $L(w,f) = 0$ 
for any measurable test function $f$. 
\end{theorem}

This theorem suggests a method 
% by which to
for estimating the MCR gradient 
using samples collected from $d_\theta$,
by first solving the maximization in~\eqref{eqn:mcr_value} to obtaining $\xi_\theta$,
% \kamyar{So, we fist compute the value using 4.} \aud{We need to calculate $\xi_\theta$ in order to get the value in (4). This is done in practice by solving the maximization problem for all $s$ first, or updating some function approximation of $\xi$.  }\kamyar{we need to first compute $h$, right?, then we get $\xi_\theta$} 
then using it to solve~\eqref{eqn:correction} to obtain
the state distribution correction $w_{\xi_\theta}$. 
The gradient can then be calculated 
using a sample-based estimate of~\eqref{eqn:ssd_grad}. 
% \kamyar{My point was on RKHS. We do not need that sentence. But if you want to keep the citation and RKHS, you should say in high dimensional settings. Not in practice.}\aud{I changed to can be. I'm hesitant to say we used $f(s) = 1$ for all $s$, for example. }
% \kamyar{not necessary. In practical setting of 4 state 4 action setting, do I need RKHS?} 
In high-dimensional settings, an RKHS kernel can be used for $f$ \citep{liu2018breaking}. 

Now that we have a method for estimating 
the the policy gradient of the MCR, 
we can use it with any PG algorithm. 
In Algorithm \ref{algo:ssd},
we propose an actor-critic algorithm 
for optimizing the MCR. 
%inspired by \citet{tamar2015policy}. 
In episode $k$, we draw trajectories 
using policy $\pi_{\theta_k}$ 
and use them to calculate 
the solutions $(\widehat{\xi}_k, \widehat{\lambda}_k)$ 
to the min-max Lagrangian problem in~\eqref{eqn:mcr_value_lagrang}. 
We then derive the state distribution corrections $\widehat{w}_{k}$ from~\eqref{eqn:correction}
using $\widehat{\xi}_k$
and calculate the gradient 
according to Theorem \ref{eqn:mcr_gradient}. 
The state distribution corrections $w$ 
and Lagrangian variables $\xi, \lambda$ 
can also be learned using neural networks
(see discussion in Appendix \ref{sec:ssd_function}).  

\begin{algorithm}%[H]
\SetAlgoLined
\KwIn{Risk envelope $\U$, learning rate $\eta$, MDP $\mathcal{M}$. }
\For{episode $k=1...K$ }{
Generate $N$ trajectories $\tau_k = \{s_0^{(n)}, a_0^{(n)}, c_0^{(n)}, ...\}_{n = 1}^N$ using $\pi_{\theta_k}$ in $\mathcal{M}$\; 
Pad $\tau_k$ if necessary\;
Solve the maximization problem in Eq. \eqref{eqn:mcr_value} to obtain $\widehat\xi_k, \widehat\lambda_k^P, \widehat\lambda_k^{\mathcal E}, \widehat\lambda_k^{\mathcal I}$ for all $(s, s')$\;
Solve the optimization problem in Eq. \eqref{eqn:correction} to obtain $\widehat{w}_k$ for all $s$\;
Compute $\widehat h_k(s,a)$ in \eqref{eqn:mcr_h_theta}\;
Compute $\nabla_\theta V_\theta = \sum_{a,s \sim \tau_k} \widehat{w}_k(s) \nabla_\theta \log \pi_{\theta_k}(a|s)\widehat{h}_k(s,a)$\;
Update $\theta_{k+1} \leftarrow \theta_k - \eta \nabla_\theta V_\theta$\;
}
\caption{State Distribution Correction Actor-Critic}\label{algo:ssd}
\end{algorithm}
% need to demonstrate variance reduction ?

\begin{figure*}
    \centering
    \begin{subfigure}[b]{0.64\textwidth}
    \includegraphics[width=1\textwidth]{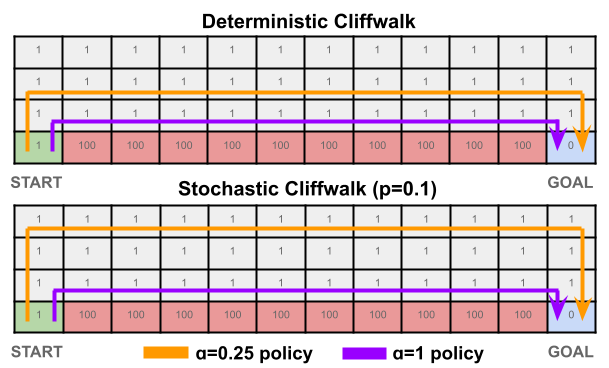}
    \caption{ } 
    \label{fig:cliff_path}
    \end{subfigure} 
    \hfill
    \begin{subfigure}[b]{0.33\textwidth}
    \includegraphics[width=1\textwidth]{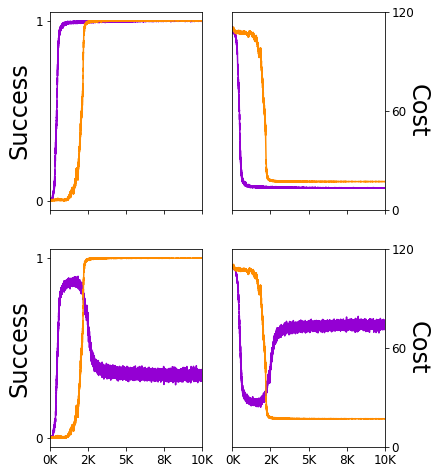}
    \caption{} 
    \label{fig:cliff_cost}
    \end{subfigure} 
    \caption{(a) $\cvara$ policies for $\alpha = 0.25$ and $\alpha=1$ learned using Algorithm \ref{algo:ssd}. (b) Average cost and success (from 1000 trajectories) over training. While the learned policy for $\alpha = 1$ always takes the shortest path, it fails to consistently reach the goal when the environment is stochastic. }\label{fig:cliffwalk}
\end{figure*}

\subsection{Convergence Result}
We provide a convergence result for Algorithm \ref{algo:ssd} to a stationary point of the MCR objective. First, we state the following assumption for our convergence guarantee: 
%, similar to that in Assumption 2 of \citet{liu2020off}:
\begin{assumption}\label{assum:stationary}
For all state-action pairs $(s,a)$ and $\theta \in \Theta$, suppose that 
\begin{enumerate}
    \item $\Big\Vert \frac{\partial \pi_\theta(a|s)}{\partial\theta}\Big\Vert \leq G$
    \item $V_\theta(s), h_\theta(s,a) \leq C_{max}$
    \item $\sigma^2_w := \E_{d_\theta}[w_{\xi_\theta}(s)^2] < \infty$
    \item the MCR is $L$-Lipschitz and $\beta$-smooth (see Corollary \ref{cor:smoothness}) 
\end{enumerate}
\end{assumption}
% \kamyar{We need to outline setting where these assumtions holds. An example would suffice.}
We have previously discussed the validity of Assumption 4 in Corollary \ref{cor:smoothness}, and Assumption 3 was previously made in \citet{liu2020off}. 
Assumption 1 can be achieved by an appropriate policy parameterization, such as through a linear or differentiable neural network function approximation with softmax output.   
Assumption 2 follows from the boundedness of $V_\theta$, an assumption common in the literature, and boundedness of the Lagrangian solutions $(\xi_\theta, \lambda)$ present in the calculation of $h_\theta$~\eqref{eqn:mcr_h_theta}. Boundedness of $\xi$ is guaranteed by the constraints in $\U$. 
The MCR objective with $\cvara$ satisfies Assumption 2, for example, because  $|\xi_\theta| \leq \frac{1}{\alpha}$ by definition, $\lambda^P_\theta$ is bounded whenever the value is bounded because it is any $(1-\alpha)$ quantile of the distribution of returns, and the remainder of the terms in $h_\theta$ are 0. 
% Similar assumptions are made in the analysis of \citet{liu2020off}. 
% Of these, assumption 1 is standard, and similar assumptions to 2 and 3 are made in \citet{liu2020off}. 
% The boundedness of $\h_\theta$ follows from the boundedness of $V_\theta$ and the solutions to the Lagrangian ~\eqref{eqn:mcr_value_lagrang} formulation of the value function.  
Following this, for Algorithm \ref{algo:ssd},
we can make the following convergence guarantee 
to a stationary point, with proof in Appendix \ref{proof:algo_stationary}. 
\begin{theorem}\label{thm:stationarity}\label{THM:STATIONARITY}
Suppose Algorithm \ref{algo:ssd} at episode $k$ and parameters $\theta_k
$ is provided with critic estimates $\widehat{V}_k$, Lagrangian variable estimates $\widehat{\xi}_k$ and $\widehat{\lambda}_k$, and distribution ratio estimates $\widehat{w}_k$ satisfying $\E_{s \sim d}(w_{\theta_k}(s) - \widehat{w}_k(s))^2 \leq \epsilon_{w,k}^2$ 
%\kamyar{who provide us with this bound, we do provide any bound here, right?} 
and $\E_{s \sim d}(h_{\theta_k}(s,a) - \widehat{h}_{\theta_k}(s,a))^2 \leq \epsilon_{h,k}^2$ for all episodes $k$. Then if Assumption \ref{assum:stationary} is satisfied, 
% \begin{align*}
%     \frac{1}{K}\sum_{k=1}^K&\E\Big[\Big\| \nabla_\theta V_{\theta_k} \Big\|^2  \Big] \leq \frac{2C_{max}}{K} + {\frac{1}{K}\sum_{k=1}^K O\Big((\epsilon_{w,k}^2C_{max}^2 + \epsilon_{h, k}^2(\sigma_w^2 + \epsilon_{w,k}^2))G^2}\Big)
\begin{align*}
    \frac{1}{K}\sum_{k=1}^K \E\Big[\Big\| \nabla_\theta V_{\theta_k} \Big\|^2  \Big] &\leq \frac{2\beta C_{max}}{\sqrt{K}} + \frac{1}{K}\sum_{k=1}^K \bigg( \frac{6}{\sqrt{K}}\Big((\epsilon_{w,k}^2C_{max}^2 + \epsilon_{h, k}^2(\sigma_w^2 + \epsilon_{w,k}^2))G^2\Big) \\
    &\quad\quad+ 2L\Big((\epsilon_{w,k}C_{max} + \epsilon_{h, k}\sqrt{\sigma_w + \epsilon_{w,k}})G\Big)\bigg) 
\end{align*}

% \end{align*}
\end{theorem}
Theorem \ref{thm:stationarity} demonstrates 
that when Assumptions \ref{assum:stationary} hold, 
the average gradient magnitude over the $K$ episodes is upper bounded by three terms. 
The first two terms vanish as $K \rightarrow\infty$, but the third does not. 
However, since the third term is in terms of $h$ and $w$ estimation errors, 
given estimators $\widehat{h}$ and $\widehat{w}$ 
with small average error $\epsilon_{h,k}, \epsilon_{w,k}$ across the episodes, 
the third term vanishes as $K \rightarrow \infty$ as well. 
Under such conditions, Theorem \ref{thm:stationarity} shows that 
Algorithm \ref{algo:ssd} will converge 
to an approximate stationary point. 
Estimators for $\widehat{h}$ and $\widehat{w}$ 
with small average error can be achieved 
with a reasonable online critic 
and $w, \xi, \lambda$ learning algorithms.
For $\cvara$, $C_{max} \propto \frac{1}{\alpha}$ due to the presence of $\xi$ in the calculation of $h_\theta$~\eqref{eqn:mcr_h_theta}, 
implying that with higher risk sensitivity, 
convergence to stationary points occurs more slowly.

\section{Experiments}\label{sec:experiments}
In this section, we examine 
the empirical efficacy of Algorithm \ref{algo:ssd} 
with two primary questions in mind:
\begin{itemize}
    \item Does the state distribution correction 
    reduce the variance of PG on MCR in MDPs? 
    \item Does increasing the risk envelope set $\U$ lead to more risk-sensitive behavior? 
\end{itemize}
\paragraph{Baseline and Implementation. }
To answer these questions, 
we optimize the MCR with the $\cvara$ objective 
at different levels $\alpha$, 
where smaller $\alpha$ corresponds 
to higher risk aversion. 
To optimize this objective,
we run Algorithm \ref{algo:ssd}
and compare its performance against 
importance sampling baselines
(Algorithm \ref{algo:is_gradient}). 
For both, we use a softmax tabular policy 
and tabular critic. 
Full implementation details 
and a function approximation version of this algorithm
are provided in Appendix \ref{appendix:hyperparams}. 

\paragraph{Simulation Domain. }
We compare the algorithms on the classic $4\times12$ tabular Cliffwalk environment, 
where the agent needs to travel 
from a start to a goal state 
while incurring as little cost as possible. 
Each action incurs 1 cost, 
but the shortest path lies next to a cliff 
and entering the cliff corresponds 
to a cost of 100.
% which incurs cost of 100. 
Because the classic Cliffwalk environment is deterministic, 
we also compare our algorithms 
on a stochastic version,
where the row of cells above the cliff is ``slippery", 
and induce a transition from these cells 
into the cliff with probability $p$. 
Although our algorithm is intended 
for an infinite-horizon MDP, 
the horizon is fixed to a maximum of 500 timesteps 
for computational feasibility. 

\paragraph{Results. } 
Figure \ref{fig:cliff_path} displays the learned policies for different CVaR levels $\alpha$  
in the deterministic and stochastic environments. 
Interestingly, the $\alpha = 0.25$ policy 
finds a conservative policy 
and takes a longer path 
in both deterministic and stochastic Cliffwalk, 
even though the optimal policy 
in the deterministic environment 
is to take the shortest path 
for all $\alpha$ levels.
In the stochastic cliffwalk, 
the $\alpha = 1$ policy 
learns to take the greedy path 
but risks falling into the cliff, 
resulting in failure to consistently succeed at the task. 
We display the learned state distribution correction 
and its discussion in Appendix~\ref{appendix:hyperparams}. 
In both environments, the gradient of 
the naive importance sampling method diverged, 
and learning a policy was not possible in either case. 
In our experiments, a policy in only 
a truncated Cliffwalk of size $2\times4$ 
was learnable via importance sampling.

\section{Discussion}\label{sec:discussion}
In this paper, we answered several open questions 
concerning PG for coherent risk functionals. 
First, we demonstrated that MCR objectives
are not in general gradient dominated, 
giving tight upper bounds 
on the suboptimality 
of the learned policy. 
Second, we proposed an algorithm 
based on stationary state distribution reweighting 
to relax stringent assumptions 
of previous work and tractably estimate the MCR gradient. 
We demonstrated that our algorithm 
is guaranteed to converge to stationary points, 
and demonstrated the stability and efficacy 
of our algorithm on the CliffWalk environment, 
in which importance sampling cannot learn policies 
due to exploding gradient magnitudes.    
% investigated the global optimality of  theoretical presented policy gradient algorithms for general coherent risk measures under the Markov coherent risk objective under both direct policy parameterization and under function approximation, and provided analysis on the global convergence
% of these algorithms. 

While coherent risk objectives 
are often used in practice, 
our results highlight the challenges 
of obtaining global optimality guarantees 
for the final policy.
One important direction of future work 
lies in improving these convergence results 
through better characterization of the residual terms 
over the learning process, 
and expanding it to other PG algorithms and settings. 
Further, the lower bound in Theorem \ref{THM:LOWER_BOUND} indicates that poorly conditioned 
starting states or initial policies 
also cause convergence to suboptimal policies.
In the future, we will investigate 
how different exploration strategies, 
such using a combination 
of expected value and coherent risk, 
may help to mitigate this issue 
both in theory and practice.

\section*{Acknowledgements}
Liu Leqi is generously supported by an Open Philanthropy AI Fellowship. Zachary Lipton thanks Amazon AI, Salesforce Research, the Block Center, the PwC Center, Abridge, UPMC, the NSF, DARPA, and SEI for supporting ACMI lab’s research on robust and socially aligned machine learning.

%\clearpage 

\bibliography{Arxiv}
\bibliographystyle{plainnat}

\newpage
\appendix

\begin{center}
{\huge Appendix}
\end{center}
\section{
%\kamyar{Proofs of statements in Section .. along with the required auxiliary lemmas} 
Proofs for Section~\ref{SEC:VALUE}}\label{sec:appendix_value}

\begin{proof}[Proof of Theorem \ref{thm:mcrp_gradient}]
The proof of this theorem is given in Appendix D of \cite{tamar2015policy}, and we repeat it here for completeness. 
The value $V_\theta$ is defined as:
\begin{align*}
    V_\theta(s) &= C(s) + \gamma \max_{\xi \in \U(P_\theta)}\sum_{s'}P_\theta(s'|s)\xi(s')V_\theta(s')\nonumber\\
    &= C(s) + \gamma \max_{\xi} \min_\lambda L_{\theta, s}(\xi, \lambda^P, \lambda^{\mathcal{I}}, \lambda^{\mathcal{E}})
\end{align*}
Using the envelope theorem \citep{milgrom2002envelope}, 
\begin{align*}
    \nabla_\theta V_\theta(s) &= \gamma \nabla_\theta \max_{\xi} \min_\lambda L_{\theta, s}(\xi, \lambda^P, \lambda^{\mathcal{I}}, \lambda^{\mathcal{E}})\nonumber\\
    &= \gamma \nabla_\theta L_{\theta, s}(\xi, \lambda^P, \lambda^{\mathcal{I}}, \lambda^{\mathcal{E}})|_{(\xi_\theta, \lambda_\theta^{P}, \lambda_\theta^{\mathcal{E}}, \lambda_\theta^{\mathcal{I}})}
\end{align*}
The derivative of the Lagrangian evaluated at the saddle point $(\xi_\theta, \lambda_\theta^{P}, \lambda_\theta^{\mathcal{E}}, \lambda_\theta^{\mathcal{I}})$ is:
\begin{align*}
    &\nabla_\theta L_{\theta, s}(\xi_{\theta, s}, \lambda_{\theta, s}^{P}, \lambda_{\theta, s}^{\mathcal{E}}, \lambda_{\theta, s}^{\mathcal{I}})\\ 
    &\quad= \sum_{s'}\xi_{\theta, s}(s')\sum_a \nabla_\theta \pi_\theta(a|s) P(s'|s,a) V_\theta(s') + \sum_{s'}P_\theta(s'|s)\xi_{\theta, s}(s')\nabla_\theta V_\theta(s') \nonumber\\
    &\quad- \lambda_{\theta, s}^{P} \sum_{s'}\xi_{\theta, s}(s')\sum_a \nabla_\theta \pi_\theta(a|s) P(s'|s,a) \\
    &\quad- \sum_{i \in \mathcal{I}} \lambda_{\theta, s}^{\mathcal{I}}(i)\frac{df_i(\xi_{\theta, s}, P_\theta)}{dp(s')}\sum_a \nabla_\theta \pi_\theta(a|s) P(s'|s,a) - \sum_{e \in \mathcal{E}} \lambda_{\theta, s}^{\mathcal{E}}(e)\frac{dg_e(\xi_{\theta, s}, P_\theta)}{dp(s')}\sum_a \nabla_\theta \pi_\theta(a|s)P(s'|s,a)
\end{align*}
Then we have 
\begin{align*}
    \nabla_\theta V_\theta(s) &= \gamma \sum_{s'}P_\theta(s'|s)\xi_{\theta, s}(s')\nabla_\theta V_\theta(s') + \sum_a  \nabla_\theta \pi_\theta(a|s)\bigg[\gamma \sum_{s'}P(s'|s,a)\Big[\xi_{\theta,s}\Big( V_\theta(s') - \lambda_{\theta, s}^{P}\Big) \\
    &-  \sum_{i \in \mathcal{I}} \lambda_{\theta, s}^{\mathcal{I}}(i)\frac{df_i(\xi_{\theta, s}, P_\theta)}{dp(s')} - \sum_{e \in \mathcal{E}} \lambda_{\theta, s}^{\mathcal{E}}(e)\frac{dg_e(\xi_{\theta, s}, P_\theta)}{dp(s')}\Big]\bigg]\\
    &= \sum_a  \nabla_\theta \pi_\theta(a|s)h_\theta(s,a) + \gamma \sum_{s'}P_\theta(s'|s)\xi_{\theta, s}(s')\nabla_\theta V_\theta(s')
\end{align*}
with $h_\theta(s,a)$ as defined in \eqref{eqn:mcr_h_theta}. Unfolding the recursion, this gives us 
\begin{align*}
    \nabla_\theta V_\theta(s) &= \sum_a  \nabla_\theta \pi_\theta(a_0|s_0)h_\theta(s_0,a_0) + \gamma \sum_{s_1}P_\theta(s_1|s_0)\xi_{\theta, s_0}(s_1)\Big( \nabla_\theta \pi_\theta(a_1|s_1)h_\theta(s_1,a_1) \\
    &\qquad+ \gamma \sum_{s_2}P_\theta(s_2|s_1)\xi_{\theta, s_1}(s_2)\nabla_\theta V_\theta(s_2) \Big)
\end{align*}
and unfolding the recursion further gives the result. 

% \kamyar{we need to unfold the recursion here, at least for one step. The way I do in PG paper, or Rich does in his PG paper.}
% I thought we already unfold for one step (we show unfolding in s_1). 
\end{proof}

\newpage 
 
%\clearpage 

\section{Proofs for Section~\ref{SEC:GRADIENT_DOMINATION}}\label{sec:appendix_gradient_domination}

\subsection{Proof of Lemma~\ref{LEM:GRADIENT_DOMINATION}}
We first establish the performance difference lemma (Lemma~\ref{lem:performance_difference}) for MCR, which upper bounds the optimality gap using an advantage function. This lemma holds for all policies, not only the directly parameterized policies we consider in Section \ref{SEC:GRADIENT_DOMINATION}. 
%Unlike in the performance difference lemma for expected value \citep{kakade2001natural}, the advantage function $A_\pi$ for MCR is a function of the primal solution $\xi_{\pi, s}$ of policy $\pi$ and the dual solution $\lambda_{*, s}$ of the optimal policy $\pi_*$ of the Lagrangian in ~\eqref{eqn:mcr_value_lagrang}.  
\begin{lemma}[Performance Difference Lemma for Markov Coherent Risk]\label{lem:performance_difference}
For all policies $\pi$ and starting state $s_0$
\begin{equation}
    V_\pi(s_0) - V_*(s_0) \leq \E_{s \sim d_{*}^{\xi_\pi}}\E_{a \sim \pi_*}\Big[A_{\pi}(s,a) + \gamma b_{\pi, s} \Big]
\end{equation}
where
%\footnotesize
\begin{align*}
    A_\pi(s,a) &= u_\pi(s) - t_\pi(s,a) \\
    t_\pi(s,a) &= \gamma \sum_{s'}P(s'|s,a)\xi_{\pi, s}(s')\Big(V_\pi(s') - \lambda_{*, s}\Big)\\
    u_\pi(s) &= \sum_a\pi(a|s)t_\pi(s,a)\\
    b_{\pi, s} &= \sum_{i \in \mathcal{I}} \lambda_{*, s}^{\mathcal{I}}(i)\Big(f_i(\xi_{\pi, s}, P_*) - f_i(\xi_{\pi, s}, P_\pi)\Big) + \sum_{e \in \mathcal{E}} \lambda_{*, s}^{\mathcal{E}}(e)\Big(g_e(\xi_{\pi, s}, P_*) - g_e(\xi_{\pi, s}, P_\pi)\Big)
\end{align*}
\normalsize
\end{lemma}
\begin{proof}
For organizational purposes, we denote 
$$f_{\pi, s} := \sum_{i \in \mathcal{I}} \lambda_{*, s}^{\mathcal{I}}(i)f_i(\xi_{\pi, s}, P_\pi) \quad \text{and} \quad f_{*, s}: = \sum_{i \in \mathcal{I}} \lambda_{*, s}^{\mathcal{I}}(i)f_i(\xi_{\pi, s}, P_*), $$
$$g_{\pi, s} := \sum_{e \in \mathcal{E}} \lambda_{*, s}^{\mathcal{E}}(e)g_e(\xi_{\pi, s}, P_\pi) \quad \text{and} \quad g_{*, s} := \sum_{e \in \mathcal{E}} \lambda_{*, s}^{\mathcal{E}}(e)g_e(\xi_{\pi, s}, P_*).$$
Then 
% \begin{equation}\label{eqn:v_id_9}
%     V_\pi(s) = C(s,a)u_\pi(s) - \gamma \lambda_{\pi, s}^{P} - \gamma f_{\pi, s} - \gamma g_{\pi, s} \quad \text{and} \quad V_*(s) = u(s) - \gamma \lambda_{*, s}^{P} - \gamma f_{*, s} - \gamma g_{*, s}
% \end{equation}
\begin{equation*}\label{eqn:b_s_id_9}
b_{\pi, s} = f_{*, s} - f_{\pi, s} + g_{*, s} - g_{\pi, s}
\end{equation*}
First we expand $V_*(s_t)$ using the definition of the Lagrangian in~\eqref{eqn:mcr_value_lagrang}, where $(\xi_{*, s_t}, \lambda_{*, s_t}^{P}, \lambda_{*, s_t}^{\mathcal{I}}, \lambda_{*, s_t}^{\mathcal{E}}) = \displaystyle\argmax_{\xi}\argmin_{\lambda^P, \lambda^{\mathcal{I}}, \lambda^{\mathcal{E}}}L_{*, s_t}(\xi, \lambda^P, \lambda^{\mathcal{I}}, \lambda^{\mathcal{E}})$ 
\allowdisplaybreaks
\begin{align*}
    V_*(s_0) - V_\pi(s_0) &= C(s_0) + \gamma L_{*, s_0}(\xi_{*, s_0}, \lambda_{*, s_0}^{P}, \lambda_{*, s_0}^{\mathcal{I}}, \lambda_{*, s_0}^{\mathcal{E}}) - V_\pi(s_0) \\
    \intertext{Because $\xi_{*, s_0}$ maximizes the Lagrangian $L_{*, s_0}$, we can lower bound using $\xi_{\pi, s_0}$, which maximizes $L_{\pi, s_0}$: }
    &\geq C(s_0) + \gamma L_{*, s_0}(\xi_{\pi, s_0}, \lambda_{*, s_0}^{P}, \lambda_{*, s_0}^{\mathcal{I}}, \lambda_{*, s_0}^{\mathcal{E}}) - V_\pi(s_0) \\
    \intertext{Expanding $L_{*, s_0}$ using~\eqref{eqn:mcr_value_lagrang}, } 
    &= \sum_{a_0} \pi_*(a_0|s_0)\Bigg(C(s_0) - \gamma \lambda_{*, s_0}\sum_{s_1}P(s_1|s_0,a_0)\xi_{\pi, s_0}(s_1) - \gamma \lambda_{*, s_0} - \gamma f_{*, s_0} - \gamma g_{*, s_0}\\
    &\quad\quad\quad\quad+ \gamma \sum_{s_1}P(s_1|s_0,a_0)\xi_{\pi, s_0}(s_1)V_*(s_1)\bigg) - V_\pi(s_0)\\
    \intertext{We can apply this Lagrangian lower bound to every nested value $V_*(s_t)$ and expand accordingly. }
    &\geq \sum_{a_0} \pi_*(a_0|s_0)\Bigg(C(s_0) - \gamma \lambda_{*, s_0}\sum_{s_1}P(s_1|s_0,a_0)\xi_{\pi, s_0}(s_1) - \gamma \lambda_{*, s_0} - \gamma f_{*, s_0} - \gamma g_{*, s_0}\\
    &\quad\quad\quad\quad+ \gamma \sum_{s_1}P(s_1|s_0,a_0)\xi_{\pi, s_0}(s_1)\bigg(\sum_{a_1} \pi_*(a_1|s_1)\Big(C(s_1) + ... \Big)\bigg) - V_\pi(s_0)\\
     \intertext{Then adding and subtracting $V_\pi(s_t)$ at each timestep $t$, and rearranging: }
    &= \sum_{a_0} \pi_*(a_0|s_0)\Bigg(C(s_0) + V_\pi(s_0) - V_\pi(s_0) - \gamma \lambda_{*, s_0}\sum_{s_1}P(s_1|s_0,a_0)\xi_{\pi, s_0}(s_1) \\
    &\quad\quad\quad\quad- \gamma \lambda_{*, s_0} - \gamma f_{*, s_0} - \gamma g_{*, s_0} + \gamma \sum_{s_1}P(s_1|s_0,a_0)\xi_{\pi, s_0}(s_1)\bigg(\sum_{a_1} \pi_*(a_1|s_1)\Big(C(s_1) \\
    &\quad\quad\quad\quad+ V_\pi(s_1) - V_\pi(s_1) + ... \Big)\bigg) - V_\pi(s_0)\\
    &= \sum_{a_0} \pi_*(a_0|s_0)\Bigg(C(s_0) - V_\pi(s_0) - \gamma \lambda_{*, s_0}\sum_{s_1}P(s_1|s_0,a_0)\xi_{\pi, s_0}(s_1) \\
    &\quad\quad\quad\quad- \gamma \lambda_{*, s_0}- \gamma f_{*, s_0} - \gamma g_{*, s_0} + \gamma \sum_{s_1}P(s_1|s_0,a_0)\xi_{\pi, s_0}(s_1)V_\pi(s_1)\\
    &\quad\quad\quad\quad+ \gamma \sum_{s_1}P(s_1|s_0,a_0)\xi_{\pi, s_0}(s_1)\bigg(\sum_{a_1} \pi_*(a_1|s_1)\Big(C(s_1) - V_\pi(s_1) + ... \Big)\bigg)\\
    \intertext{Using the definition of $d_*^{\xi_\pi}(s)$, }
    &= \sum_s d_*^{\xi_\pi}(s) \sum_a  \pi_*(a|s) \bigg[C(s) +  \gamma \sum_{s'}P(s'|s,a)\xi_{\pi, s}(s')V_\pi(s') \\
    &\quad\quad\quad\quad- \gamma \lambda_{*, s}\sum_{s'}P(s'|s,a)\xi_{\pi, s}(s') - \gamma \lambda_{*, s} - \gamma f_{*, s} - \gamma g_{*, s}  - V_\pi(s)\bigg]\\
    \intertext{Then  substituting~\eqref{eqn:mcr_value_lagrang} for $V_\pi(s)$, }
    &= \sum_s d_*^{\xi_\pi}(s) \sum_a  \pi_*(a|s) \bigg[C(s) +  \gamma \sum_{s'}P(s'|s,a)\xi_{\pi, s}(s')V_\pi(s') \\
    &\quad\quad\quad\quad- \gamma \lambda_{*, s}\sum_{s'}P(s'|s,a)\xi_{\pi, s}(s') - \gamma \lambda_{*, s_0} - \gamma f_{*, s} - \gamma g_{*, s}  \\
    &\quad\quad\quad\quad - \Big(C(s) + \gamma L_{\pi, s}( \xi_{\pi, s}, \lambda_{\pi, s}^{P}, \lambda_{\pi, s}^{\mathcal{I}}, \lambda_{\pi, s}^{\mathcal{E}} )\Big)\bigg]\\
    \intertext{Because $\lambda_{\pi, s}$ minimizes the $L_{\pi, s}$, we lower bound using $\lambda_{*, s}$:}
    &\geq \sum_s d_*^{\xi_\pi}(s) \sum_a  \pi_*(a|s) \bigg[C(s) +  \gamma \sum_{s'}P(s'|s,a)\xi_{\pi, s}(s')V_\pi(s') \\
    &\quad\quad\quad\quad- \gamma \lambda_{*, s}\sum_{s'}P(s'|s,a)\xi_{\pi, s}(s') - \gamma \lambda_{*, s_0} - \gamma f_{*, s} - \gamma g_{*, s}  \\
    &\quad\quad\quad\quad - \Big(C(s) + \gamma L_{\pi, s}( \xi_{\pi, s}, \lambda_{*, s}^{P}, \lambda_{*, s}^{\mathcal{I}}, \lambda_{*, s}^{\mathcal{E}} )\Big)\bigg]\\
    \intertext{Finally, expanding the Lagrangian and grouping terms,}
    &= \sum_s d_*^{\xi_\pi}(s) \sum_a  \pi_*(a|s) \bigg[ \gamma \sum_{s'}P(s'|s,a)\xi_{\pi, s}(s')V_\pi(s') - \gamma \lambda_{*, s}\sum_{s'}P(s'|s,a)\xi_{\pi, s}(s') \\
    &\quad\quad\quad\quad\quad- \gamma \lambda_{*, s}- \gamma f_{*, s} - \gamma g_{*, s}  - \Big(u_\pi(s) - \gamma \lambda_{*, s} - \gamma f_{\pi, s} - \gamma g_{\pi, s}\Big)\bigg]\\
    &= \sum_s d_*^{\xi_\pi}(s) \sum_a  \pi_*(a|s) \bigg[t_\pi(s,a) - u_\pi(s) - \gamma b_{\pi, s} \bigg]\\
    % &= \sum_s d_{s_0, \xi_\pi}(s) \sum_a  \pi_*(a|s) \bigg[A^\theta(s,a) + \gamma b_s \bigg]
\end{align*}
From the above, flipping the direction of the inequality we obtain 
\begin{equation*}
    V_\pi(s_0) - V_*(s_0) \leq \sum_s d_*^{\xi_\pi}(s)  \sum_a  \pi_*(a|s) \Big(u_\pi(s) - t_\pi(s,a) + \gamma b_{\pi, s} \Big)
\end{equation*}
Rearranging and using the definition of $A_\pi(s,a)$ gives the result. 
\end{proof}

%\paragraph{Proof of Lemma \ref{lem:gradient_domination}}
\begin{proof}[Proof of Lemma \ref{LEM:GRADIENT_DOMINATION}]
We now give a proof for the gradient domination lemma for directly parameterized policies, and going forward we interchangeably refer to the policy as $\pi$ and $\theta$. From the performance difference lemma \ref{lem:performance_difference}, we have 
\begin{align*}
    V_\theta(s_0) - V_*(s_0) &\leq \sum_s d_*^{\xi_\theta}(s) \Big(\sum_a \pi_*(a|s) A_\theta(s,a)  + \gamma b_{\theta, s} \Big)\\
    \intertext{Using the fact that $\sum_a \pi_\theta(a|s)A_\theta(s,a) = 0$, }
    V_\theta(s_0) - V_*(s_0) &\leq \sum_s d_*^{\xi_\theta}(s)\sum_a \Big(\pi_*(a|s) - \pi_\theta(a|s)\Big) A_\theta(s,a)  + \gamma b_{\theta, s} \\
    \intertext{Since $\sum_a(\pi(a|s) - \pi'(a|s))u_\theta(s) = 0$ for any two policies $\pi, \pi'$ }
    &= \sum_s d_*^{\xi_\theta}(s) \sum_a \Big(\pi_\theta(a|s) - \pi_*(a|s)\Big) t_\theta(s,a)  + \gamma b_{\theta, s} \\
\end{align*}
Under direct parameterization where $\theta_{s,a} = \pi_\theta(a|s)$, the gradient from Theorem \eqref{eqn:mcr_gradient} is 
    $$\frac{\partial V_\theta(s_0)}{\partial \theta_{s,a}} = d_\theta^{ \xi_\theta}(s)h_\theta(s,a)$$
Due to the nonlinearity in the objective, while $t_\theta$ contains $\lambda_{*, s}$, $h_\theta$ is defined using $\lambda_{\theta, s}$. Further, $h_\theta$ involves the derivatives of the terms in $b_{\theta, s}$. Rearranging the previous inequality, 
% \begin{align*}
%      \text{res}_{\theta, s} &= \Big(\lambda_{\theta, s}^{P} - \lambda_{*, s}^{P}\Big)\sum_{a}\Big(\pi_\theta(a|s) - \pi_*(a|s)\Big) \sum_{s'}P(s'|s,a)\xi_{\theta, s}(s')\\
%      \text{lin}_{\theta, s} &= \gamma\bigg[\sum_{a}\Big(\pi_\theta(a|s) - \pi_*(a|s) \Big)\sum_{s'} P(s'|s,a) \Big(\sum_{i} \lambda_{\theta, s}^{\mathcal{I}}(i)\frac{df_{i}(\xi_{\theta, s}, P_\theta)}{dp(s')} + \sum_{e} \lambda_{\theta, s}^{\mathcal{E}}(e)\frac{dg_{e}(\xi_{\theta, s}, P_\theta)}{dp(s')}\Big) + b_{\theta, s}\bigg]\\
% \end{align*}
\begin{align*}
    V_\theta(s_0) - V_*(s_0) &\leq \sum_s d_*^{\xi_\theta}(s) \sum_a \Big(\pi_\theta(a|s) - \pi_*(a|s)\Big) t_\theta(s,a) + \gamma b_{\theta, s}\\
    &\leq \max_{\overline{\pi} \in \Delta(\A)^{\Ss}}\sum_s d_*^{\xi_\theta}(s) \sum_a \Big(\pi_\theta(a|s) - \overline{\pi}(a|s)\Big)t_\theta(s,a) + \gamma b_{\theta, s}\\
    &\leq \Big\Vert \frac{d_*^{\xi_\theta}}{d_\theta^{\xi_\theta}} \Big\Vert_\infty  \max_{\overline{\pi} \in \Delta(\A)^{\Ss}}\sum_s d_*^{\xi_\theta}(s) \sum_a \Big(\pi_\theta(a|s) - \overline{\pi}(a|s)\Big) t_\theta(s,a) + \gamma b_{\theta, s}\\
    &\leq \Big\Vert \frac{d_*^{\xi_\theta}}{d_\theta^{\xi_\theta}} \Big\Vert_\infty  \bigg(\max_{\overline{\pi} \in \Delta(\A)^{\Ss}}\sum_s d_*^{\xi_\theta}(s) \sum_a \Big(\pi_\theta(a|s) - \overline{\pi}(a|s)\Big) h_\theta(s,a) + \epsilon_L(\theta, \pi_*) + \epsilon_U(\theta, \pi_*)\bigg)\\
    &= \Big\Vert \frac{d_*^{\xi_\theta}}{d_\theta^{\xi_\theta}} \Big\Vert_\infty  \bigg(\max_{\overline{\pi} \in \Delta(\A)^{\Ss}}\sum_s d_*^{\xi_\theta}(s)  \Big(\pi_\theta - \overline{\pi}\Big)^\top \nabla_\theta V_\theta + \epsilon_L(\theta, \pi_*) + \epsilon_U(\theta, \pi_*)\bigg)\\
\end{align*}
where the second to last line uses the residual definitions~\eqref{eqn:policy_res} and~\eqref{eqn:constraint_res}, and the last line uses the definition of the gradient. 

\end{proof}
%\clearpage 

\subsection{Proof of Theorem \ref{THM:LOWER_BOUND}}\label{proof:lower_bound}

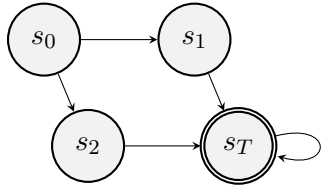
\begin{figure}[h]
    \centering
    \begin{tikzpicture}
        \node[state, ] (q1) {$s_0$};
        \node[state, right of=q1] (q2) {$s_1$};
        \node[state, below left of=q2] (q3) {$s_2$};
        \node[state, accepting, right of=q3] (q4) {$s_T$};
        \draw 
                (q1) edge[above] node{} (q2)
                (q1) edge[above] node{} (q3)
                % (q2) edge[loop right] node{0} (q2)
                % (q3) edge[loop right] node{0} (q3)
                (q3) edge[above] node{} (q4)
                (q2) edge[above] node{} (q4)
                (q4) edge[loop right] node{} (q4);
                % (q2) edge[loop above] node{1} (q2)
                % (q2) edge[bend left, above] node{0} (q3)
                % (q3) edge[bend left, below] node{0, 1} (q2);
    \end{tikzpicture}
    \caption{The MDP for the proof of Theorem~\ref{thm:lower_bound}}
    \label{fig:apdx_lower_bound_mdp}
\end{figure}
Consider the MDP of Figure \ref{fig:apdx_lower_bound_mdp}.
The MDP has the following transitions: 
\begin{align*}
    &P(s_2 | s_0, a_1) = 1\\
    &P(s_1 | s_0, a_2) = 0.9\\
    &P(s_2 | s_0, a_2) = 0.1\\
    &P(s_T|s_1) = P(s_T|s_2) = 1
\end{align*}
and the following costs: 
\begin{equation*}
    C(s_0) = 0, \quad C(s_1) = 0,\quad C(s_2) = 1, \quad C(s_T) = 0
\end{equation*}
There are 4 states, with $s_T$ being absorptive, and only 2 actions. The policy $\pi_\theta$ affects only the transitions from the starting state $s_0$ to either of the next states $s_1$ or $s_2$. Note that this MDP follows the same dynamics as a binomial 2-armed bandit, where arm 1 always returns cost 1 and arm 2 returns a cost of 0 90\% of the time and 10\% of the time returns cost 1.
% put MDP figure & explanation here 
Suppose we work with direct parameterization, where $\pi_\theta(a_1|s) = \theta_{a_1, s}$ and $\pi_\theta(a_2, s) = 1 - \theta_{a_2, s}$. We limit our consideration to $\theta_{a_1, s_0} = \theta$ and $\theta_{a_2, s_0} = 1 - \theta$ to avoid overparameterizing our policy, as the policy does not affect the cost accrued after $s_0$. Let us consider $\alpha = 0.5$, the expectation over the \emph{highest} 50\% of the distribution of the total discounted returns. We will show that there exists $\theta$ in this MDP such that the gradient domination bound in Lemma \ref{LEM:GRADIENT_DOMINATION} is tight, thus giving Theorem \ref{THM:LOWER_BOUND}. 

We can deterministically calculate the Markov coherent risk value of states $s_1, s_2$ using the definition \eqref{eqn:mcr_value}:
\begin{align*}
    V^\text{MC}(s_1) &= 0\\
    V^\text{MC}(s_2) &= 1\\
\end{align*}

\paragraph{Optimal Policy $\theta_*$: } The optimal policy in the setting is $\theta = 0$, which always takes action $a_2$. 

\paragraph{Policy $\theta$: } Set $\theta = \frac{4}{9}$.

The value of the initial state $V^\text{MC}$ can then be calculated using \eqref{eqn:mcr_value}. The maximization problem can be solved by either
% \begin{equation*}
%     V_\theta(s_0) = \gamma \max_{\xi \in \U}\Big(\theta \xi(s_1) V(s_1) + (1-\theta)\xi(s_2)V(s_2)\Big)
% \end{equation*}
% where for $\text{CVaR}_{0.5}$ the risk envelope $\U$ is
% \begin{equation*}
%     \U = \left\{\xi : \theta \xi(s_1) + (1-\theta)\xi(s_2) = 1, \xi \in [0, 2]\right\}
% \end{equation*}
using the analytical solution to the CVaR risk envelope or by convex programming solver, which gives us the primal solution $\xi_{s_0}$ and dual solution $\lambda_{s_0}$ for $\theta_*$ and $\theta$, which are shown in the table below:
\begin{center}
\begin{tabular}{ |c|c|c|c|c| } 
\hline
      & $\xi_{s_0}(s_1)$ & $\xi_{s_0}(s_2)$ & $\lambda^{P}$ & $V^\text{MC}(s_0)$ \\
      \hline
    $\theta$ & 0 & 2 & 1 & 1 \\ 
     \hline
    $\theta_*$ & 8/9 & 2 & 0 & 0.2 \\ 
     \hline
\end{tabular}
\end{center}
The optimality gap is 
\begin{equation*}
    V_\theta(s_0) - V_*(s_0) = 0.8\gamma
\end{equation*}
We will show that the bound in Theorem \ref{lem:gradient_domination} is tight for $\theta$, and proceed by determining the terms on the RHS of the bound. 

First, we calculate the gradient $\nabla_\theta V_\theta(s_0)$ using Theorem \ref{thm:mcrp_gradient}. Under direct parameterization in our MDP, we have that 
\begin{equation*}
    \frac{\partial V_\theta(s_0)}{d\theta_{s_0, a}} = h_\theta(s_0, a). 
\end{equation*}
Using the Lagrangian solutions from the table, we obtain that 
\begin{align*}
    h_\theta(s_0, a_1) = h_\theta(s_0, a_2) = 0 
\end{align*}
which indicates that the magnitude of the gradient is 0. Next, we calculate the residual term $\epsilon_{\text{res}}$ using the Lagrangian solutions:  
\begin{align*}
    &\gamma\Big(\lambda_{\theta, s_0}^{*,P} - \lambda_{*, s_0}^{*,P}\Big)\sum_a \Big(\theta_{s_0, a} - \theta_{*, s_0,a}\Big)\sum_{s'} P(s'|s,a)\xi_{\theta, s_0}(s') = 0.8\gamma.
\end{align*}
The last residual term satisfies  $\epsilon_{U} = 0$ because the constraint expressions $f_i$ are independent of $\theta$. Then we have the desired result, i.e., 
\begin{equation*}
    V_\theta(s_0) - V_*(s_0) = \epsilon_{L}(\theta, \pi_*) = 0.8\gamma. 
\end{equation*}

\subsection{Proof of Proposition \ref{thm:pgd}}\label{proof:pgd}
First, we restate Proposition B.1 from \cite{agarwal2019theory}, built upon results from \citet{ghadimi2016accelerated}, without proof, which we will later use in the proof of our theorem. 
\begin{proposition}[Proposition B.1 from \cite{agarwal2019theory}]\label{prop:b_1}
Let $V_\theta(s_0)$ be $\beta$-smooth in $\theta$ and the update rule is $\theta^+ = \theta - \eta G^\eta(\theta)$. If $\Vert V^\eta \Vert_2 \leq \epsilon$ then 
\begin{equation*}
    \max_{\overline{\pi} \in \Delta(\A)^{|\Ss|}}(\overline{\pi} - \pi)^\top\nabla_\theta V_\theta(s_0) \leq \epsilon\Big(\eta\beta + 1\Big)
\end{equation*}
\end{proposition}
%Now we prove Theorem \ref{thm:pgd}.

\begin{proof}[Proof of Proposition \ref{thm:pgd}]
 Let
 $$G^\eta(\theta) = \frac{1}{\eta}\Big(\theta - P_{\Delta(\A)^{\Ss}}(\theta - \eta \nabla_\theta V_\theta(s_0))\Big)$$
Let $\theta_t$, $V_t$ respectively be the parameter and value parameterized by $\theta_t$ at time $t$. Using our assumption that $V_\theta(s_0)$ is $\beta$-smooth, and from \cite{beck2017first} we have that for step size $\eta = \frac{1}{\beta}$, after $T$ steps of PGD, 
\begin{align*}
    \min_t \Vert G^\eta(\theta_t)\Vert_2 \leq \sqrt{\frac{2\beta(V_0(s_0) - V_*(s_0))}{T}}
\end{align*}
% $$\min_t \Vert G^\eta(\theta_t)\Vert_2 \leq \sqrt{\frac{2\beta(V_0(s_0) - V_*(s_0))}{T}}$$
Then using Proposition \ref{prop:b_1} with $\eta = \sqrt{\frac{2\beta(V_0(s_0) - V_*(s_0))}{T}}$, 
$$\min_t \max_{\theta_t + \delta \in \Delta(\A)^{\Ss}, \Vert \delta \Vert_2 \leq 1} \delta^\top \nabla_\theta V_{\theta_{t+1}}(s_0) \leq (\eta \beta + 1)\sqrt{\frac{2\beta(V_0(s_0) - V_*(s_0))}{T}}$$
Because $\Vert \overline{\pi} - \pi \Vert_2 \leq 2\sqrt{|\Ss|}$, 
\begin{align*}
    \max_{\overline{\pi} \in \Delta(\A)^{|\Ss|}}(\overline{\pi} - \pi)^\top\nabla_\theta V_\theta(s_0) &= 2\sqrt{|\Ss|}\max_{\overline{\pi} \in \Delta(\A)^{|\Ss|}}\frac{1}{\sqrt{2|\Ss|}}(\overline{\pi} - \pi)^\top\nabla_\theta V_\theta(s_0)\\
    &\leq  2\sqrt{|\Ss|}\max_{\theta + \delta \in \Delta(\A)^{\Ss}, \Vert \delta \Vert_2 \leq 1}\delta^\top \nabla_\theta V_\theta(s_0)
\end{align*}
Then using the gradient domination lemma (Lemma \ref{lem:gradient_domination}) and the fact that $\eta\beta = 1$,
\begin{align*}
    \min_t V_{t}(s_0) - V_*(s_0) &\leq \max_t\Big\Vert \frac{d_*^{\xi_{\theta_t}}}{d_{\theta_t}^{\xi_{\theta_t}}}\Big\Vert_\infty 4\sqrt{|\Ss|}\sqrt{\frac{2\beta(V_0(s_0) - V_*(s_0))}{T}} +  \max_t \Big(\epsilon_{L}(\theta_t, \pi_*) + \epsilon_{U}(\theta_t, \pi_*)\Big) \\
    &\leq \max_t\Big\Vert \frac{d_*^{\xi_{\theta_t}}}{d_{\theta_t}^{\xi_{\theta_t}}}\Big\Vert_\infty \sqrt{\frac{32|\Ss|\beta}{(1-\gamma)T}} + \epsilon_L + \epsilon_U 
\end{align*}
where the last line uses Assumptions \ref{assum:eps_l} and \ref{assum:eps_u} and the fact that $V_{0}(s_0) - V_*(s_0) \leq \frac{1}{1-\gamma}$.  
% We can get our required bound that the error on the RHS is $\leq \epsilon$ if we set 
% $$T \geq \max_t \Big\Vert \frac{d_{s_0,\xi_{t}}}{d^\theta_{s_0,\xi_{t}}}\Big\Vert_\infty^2 32|\Ss|\frac{\beta(V_{(0)}(s_0) - V_*(s_0))}{(\epsilon - \text{res} - \text{lin})^2}$$
\end{proof}
%\newpage 

\begin{lemma}[Smoothness of $V_\theta$ for Direct Parameterization]\label{lem:smoothness}
If the gradients of the primal and dual solutions of the Markov coherent risk value function (\ref{eqn:mcr_value}) exist and are bounded everywhere, that is 
%\scriptsize
$$\Big\Vert \frac{d\xi_\theta}{d\theta}\Big\Vert_\infty \leq c\quad \text{and}\quad \Big\Vert \frac{d\lambda_\theta}{d\theta}\Big\Vert_\infty \leq c'$$ 
\normalsize
and the second derivative of the constraints exist and are bounded, that is for all $e \in \mathcal{E}$ and $i \in \mathcal{I}$
%\scriptsize
$$\Big\Vert \frac{d^2g_e(\xi, P)}{dP^2}  \Big\Vert \leq M' \quad\text{and}\quad \Big\Vert \frac{d^2f_i(\xi, P)}{dP^2}  \Big\Vert \leq M'$$
\normalsize
then there exists $\beta$ such that the value $V_\theta$ is $\beta$-smooth for starting state $s_0$, that is 
\begin{equation*}
    \Vert \nabla_\theta V_\theta(s_0) - \nabla_\theta V_{\theta'}(s_0)\Vert_2 \leq \beta \Vert\theta - \theta'\Vert_2
\end{equation*}
\end{lemma}

\begin{proof}
Let $\pi_\alpha := \pi_{\theta + \alpha u}$ where $u$ is a unit vector in $\mathbb{R}^d$, and let $V_\alpha$ denote the value function under policy $\pi_\alpha$. 
Note that our value function is 
\begin{align*}
    V_\alpha(s) &= C(s) + \max_{\xi \in \U(P_\alpha)}\sum_{s'}P_\alpha(s'|s)\xi(s')V_\alpha(s')\\
    &= C(s) + \max_{\xi}\min_{\lambda^P, \lambda^{\mathcal{E}}, \lambda^{\mathcal{I}}}L_{\alpha, s}(\xi, \lambda^P, \lambda^{\mathcal{E}}, \lambda^{\mathcal{I}})\\
\end{align*}
and $\forall s \in \Ss$, 
let $(\xi_{\alpha, s}, \lambda_{\alpha, s}^{P}, \lambda_{\alpha, s}^{\mathcal{E}}, \lambda_{\alpha, s}^{\mathcal{I}})$ denote the saddle points of the Lagrangian $L_{\alpha, s}$ under policy $\pi_\alpha$.% To prove the lemma statement, w
We will prove that the second derivative $\frac{d^2V_\alpha}{d^2\alpha}$ is bounded above by the smoothness parameter $\beta$.  
Let $\widetilde{P}(\alpha) \in \mathbb{R}^{|\Ss|\times|\Ss|}$ be the state-state transition matrix under $\pi_\alpha$
\begin{equation*}
    [\widetilde{P}(\alpha)]_{s\rightarrow s'} =  \sum_a \pi_\alpha(a|s) P(s'|s,a)
\end{equation*}
Further, define the $\xi_\alpha$-reweighted transition to be 
\begin{equation*}
    [\widetilde{P}_{\xi}(\alpha)]_{s \rightarrow s'} = [\widetilde{P}(\alpha)\xi_\alpha]_{ s\rightarrow s'} = \xi_{\alpha, s}(s')\sum_a\pi_\alpha(a|s)P(s'|s,a). 
\end{equation*}
For notation simplicity, let 
\begin{equation*}
    y_\alpha(s) = \sum_e \lambda_{\alpha, s}^{\mathcal{E}}(e) \frac{dg_e(\xi_{\alpha, s}, P_\alpha)}{dP} + \sum_i \lambda_{\alpha, s}^{\mathcal{I}}(i) \frac{df_i(\xi_{\alpha, s}, P_\alpha)}{dP}
\end{equation*}
Under our assumptions, we have that 
\begin{enumerate}
    \item $\sum_a\Big|\frac{d\pi_\alpha(a|s)}{d\alpha}\Big|_{\alpha=0}\Big| \leq C_1$
    \item $\sum_a\Big|\frac{d^2\pi_\alpha(a|s)}{d^2\alpha}\Big|_{\alpha=0}\Big| \leq C_2$
    \item $\Vert \xi_{\alpha}\Vert_\infty \leq C_3$ (for $\text{CVaR}_{\alpha}$, $C_3 = 1/\alpha$)
    \item $\Vert \lambda_{\alpha}\Vert_\infty \leq C_4$ (for CVaR, $C_4 = 1$)
    \item  $\Vert \frac{d\xi_{\alpha}}{d\alpha}\Vert_\infty \leq C_5$
    \item  $\Vert \frac{d\lambda_{\alpha}^{*}}{d\alpha}\Vert_\infty \leq C_6$ 
    for all dual solutions $\lambda_\alpha$
    \item $\| \frac{dy_\alpha}{d\alpha}\|_\infty \leq C_7$ as a consequence of 6. and the assumption that the second derivative of the constraints is bounded. 
\end{enumerate}
We have that for an arbitrary vector $x$ that
\begin{align*}
    \displaystyle\max_{\Vert u \Vert_2 = 1}\bigg|\bigg[ \frac{d\widetilde{P}(\alpha)}{d\alpha}x \bigg]_{s}\bigg| &\leq C_1 \Vert x \Vert_\infty \\
     \displaystyle\max_{\Vert u \Vert_2 = 1}\bigg|\bigg[ \frac{d^2\widetilde{P}(\alpha)}{d\alpha^2}x \bigg]_{s}\bigg| &\leq C_2 \Vert x \Vert_\infty 
\end{align*}
We also have 
\begin{align*}
    \displaystyle\max_{\Vert u \Vert_2 = 1}\bigg|\bigg[ \frac{d\widetilde{P}_{\xi}(\alpha)}{d\alpha}x \bigg]_{s}\bigg| &= \max_{\Vert u \Vert_2 = 1} \bigg|\sum_{a,s'}\Big(\frac{d\pi_\alpha(a|s')}{d\alpha}\xi_{\alpha, s}(s') + \pi_\alpha(a|s')\frac{d\xi_{\alpha, s}(s')}{d\alpha} \Big)\Big|_{\alpha=0}P(s'|s,a)x_{a,s'}  \bigg|\\
    &\leq \sum_{a,s} \Big|\frac{d\pi_\alpha(a|s)}{d\alpha}\xi_{\alpha, s}\Big|_{\alpha=0}\Big|P(s'|s,a)|x_{s',a}| + \sum_{a,s} \Big| \pi_\alpha(a|s)\frac{d\xi_{\alpha, s}(s')}{d\alpha}\Big|_{\alpha=0}\Big|P(s'|s,a)|x_{s,a}|\\
    % &\leq \sum_{s'}P(s'|s,a)\Big|\xi_{\alpha, s}(s')\Big|_{\alpha=0}\Big|\Vert x \Vert_\infty\sum_{a}\Big|\frac{d\pi_\alpha(a|s)}{d\alpha}\Big|_{\alpha=0}\Big| \\
    % &\qquad+ \sum_{s'}P(s'|s,a)\Big|\frac{d\xi_{\alpha, s}(s')}{d\alpha}\Big|_{\alpha=0}\Big| \Vert x \Vert_\infty \sum_{a}\pi_\theta(a|s)\\
    &\leq (C_1 C_3 + C_5) \Vert x \Vert_\infty 
\end{align*}
and 
\begin{align*}
    \displaystyle\max_{\Vert u \Vert_2 = 1}\bigg|\bigg[ \frac{d\widetilde{P}(\alpha)}{d\alpha}\frac{d\xi_\alpha}{d\alpha}x \bigg]_{s}\bigg| &= \max_{\Vert u \Vert_2 = 1} \bigg| \sum_{a,s'} \frac{d\pi_\alpha(a|s')}{d\alpha}\frac{d\xi_{\alpha, s}(s')}{d\alpha}\Big|_{\alpha=0}P(s'|s,a)x_{a,s'} \bigg|\\
    &\leq \sum_{a,s}\Big|\frac{d\pi_\alpha(a|s)}{d\alpha}\Big|_{\alpha=0} \Big|\Big|\frac{d\xi_{\alpha, s}(s')}{d\alpha}\Big|_{\alpha=0}\Big|P(s'|s,a)|x_{s,a}\\
    % &\leq \sum_{s'}\Big|\frac{d\xi_{\alpha, s}(s')}{d\alpha}\Big|_{\alpha=0}\Big|P(s'|s,a)\Vert x \Vert_\infty\sum_{a}\Big|\frac{d\pi_\alpha(a|s')}{d\alpha}\Big|_{\alpha=0}\Big|\\
    &\leq C_1 C_5 \Vert x \Vert_\infty .
\end{align*}
We can write $V_\alpha(s_0)$ as
\[
    V_\alpha(s_0) = e^\top_{s_0}\Big(\textbf{I} - \gamma \widetilde{P}_{\xi}(\alpha)\Big)^{-1}C,
\]
where $C \in \R^{|\Ss|}$ is a vector whose $s$-th entry is the cost of state $s$. 
Using the envelope theorem, 
\begin{align*}
    \frac{d}{d\alpha}V_\alpha(s) &= \frac{d}{d\alpha}L_{\alpha, s}(\xi, \lambda)|_{(\xi_{\alpha, s}, \lambda_{\alpha, s})}\\
    &= \gamma  \Big(\sum_{s'}\frac{dP_\alpha(s'|s)}{d\alpha}\xi_{\alpha, s}(s')V_\alpha(s') +  \sum_{s'}P_\alpha(s'|s)\xi_{\alpha, s}(s')\frac{dV_\alpha(s')}{d\alpha} \\
    &\qquad- \lambda_{\alpha, s}^{P}\sum_{s'}\frac{dP_\alpha(s'|s)}{d\alpha}\xi_{\alpha, s}(s') - \sum_{s'}\frac{dP_\alpha(s'|s)}{d\alpha} y_\alpha(s)\Big). 
\end{align*}
Writing this in matrix form and substituting in our previous expression for $V_\alpha$, we get 
\[
    \frac{dV_\alpha(s)}{d\alpha} = \gamma e^\top_{s}\Big(\textbf{I} - \gamma \widetilde{P}_{\xi}(\alpha)\Big)^{-1}\bigg( \frac{d\widetilde{P}(\alpha)}{d\alpha}\xi_\alpha  \Big(\textbf{I} - \gamma \widetilde{P}_{\xi}(\alpha)\Big)^{-1}C - \lambda_{\alpha}^{P}\frac{d\widetilde{P}(\alpha)}{d\alpha}\xi_{\alpha} - \frac{d\widetilde{P}(\alpha)}{d\alpha} y_\alpha \bigg).
\]
Next, we take the second derivative using the chain rule: 
\begin{align*}
    &\frac{d^2V_\alpha(s)}{d\alpha^2} \\
    &= \gamma^2e^\top_{s}\Big(\textbf{I} - \gamma \widetilde{P}_{\xi}(\alpha)\Big)^{-1} \frac{d\widetilde{P}_{\xi}(\alpha)}{d\alpha}\Big(\textbf{I} - \gamma \widetilde{P}_{\xi}(\alpha)\Big)^{-1}\bigg( \frac{d\widetilde{P}(\alpha)}{d\alpha}\xi_\alpha  \Big(\textbf{I} - \gamma \widetilde{P}_{\xi}(\alpha)\Big)^{-1}C - \lambda_{\alpha}^{P}\frac{d\widetilde{P}(\alpha)}{d\alpha}\xi_{\alpha} - \frac{d\widetilde{P}(\alpha)}{d\alpha} y_\alpha \bigg)\\
    &\qquad+ \gamma e^\top_{s}\Big(\textbf{I} - \gamma \widetilde{P}_{\xi}(\alpha)\Big)^{-1}\bigg(  \frac{d^2\widetilde{P}(\alpha)}{d\alpha^2}\xi_\alpha \Big(\textbf{I} - \gamma \widetilde{P}_{\xi}(\alpha)\Big)^{-1}C + \frac{d\widetilde{P}(\alpha)}{d\alpha}\frac{d\xi_\alpha}{d\alpha}\Big(\textbf{I} - \gamma \widetilde{P}_{\xi}(\alpha)\Big)^{-1}C \\
    &\qquad+ \gamma\frac{d\widetilde{P}(\alpha)}{d\alpha}\xi_\alpha \Big(\textbf{I} - \gamma \widetilde{P}_{\xi}(\alpha)\Big)^{-1} \frac{d\widetilde{P}_{\xi}(\alpha)}{d\alpha}\Big(\textbf{I} - \gamma \widetilde{P}_{\xi}(\alpha)\Big)^{-1}C\\
    &\qquad- \frac{d\lambda_{\alpha}^{P}}{d\alpha}\frac{d\widetilde{P}(\alpha)}{d\alpha}\xi_{\alpha} - \lambda_{\alpha}^{P}\Big(\frac{d^2\widetilde{P}(\alpha)}{d\alpha^2}\xi_\alpha +  \frac{d\widetilde{P}(\alpha)}{d\alpha}\frac{d\xi_\alpha}{d\alpha}\Big)\\
    &\qquad- \frac{d^2\widetilde{P}(\alpha)}{d\alpha^2}y_\alpha - \frac{d\widetilde{P}(\alpha)}{d\alpha}\frac{dy_\alpha}{d\alpha}
    \bigg)
\end{align*}
Let $M(\alpha) = \Big(\textbf{I} - \gamma \widetilde{P}_{\xi}(\alpha)\Big)^{-1}$. By using power series expansion of matrix inverse, we can write $M(\alpha)$ as
\[
    M(\alpha) = \sum_{n=0}^\infty \gamma^n\Big(\widetilde{P}(\alpha)\xi_\alpha\Big)^n.
\]
This implies that $M(\alpha) \geq 0$ componentwise and each row of $M(\alpha)$ sums to $\frac{1}{1-\gamma}$, that is $M(\alpha)\textbf{1} = \frac{1}{1-\gamma}\textbf{1}$. This implies 
\[
    \max_{\Vert u \Vert_2 = 1}\Vert M(\alpha) x\Vert_\infty \leq \frac{1}{1-\gamma}\Vert x \Vert_\infty
\]
Then, we have  
\begin{align*}
    \max_{\Vert u \Vert_2 = 1} \bigg| \frac{dV_\alpha(s)}{d\alpha}\Big|_{\alpha = 0} \bigg| &\leq 
    \gamma \Big\Vert M(\alpha)\frac{d\widetilde{P}(\alpha)}{d\alpha}\xi_\alpha M(\alpha) C \Big\Vert_\infty 
    + \gamma \Big\Vert M(\alpha)y_\alpha  \Big\Vert_\infty 
    + \gamma \Big\Vert \lambda_{\alpha}^{P}M(\alpha)\xi_\alpha \Big\Vert_\infty 
    \\
    &\leq \frac{\gamma}{(1-\gamma)^2} C_{max}\Big\Vert \frac{d\widetilde{P}(\alpha)}{d\alpha}\xi_\alpha  \Big\Vert_\infty + \frac{\gamma}{1-\gamma}\Big\Vert y_\alpha  \Big\Vert_\infty + \frac{\gamma}{1-\gamma}C_4\Big\Vert \xi_\alpha \Big\Vert_\infty\\
    \intertext{Note that by Assumption \ref{assum:risk_envelope}, $y_\alpha$ is upper bounded by $\Big(|\mathcal{E}| + |\mathcal{I}|\Big)M C_4$, }
    &\leq \frac{\gamma}{(1-\gamma)^2} C_1C_3C_{max} + \frac{\gamma}{1-\gamma}\Big(|\mathcal{E}| + |\mathcal{I}|\Big)MC_4 + \frac{\gamma}{1-\gamma}C_1C_3.
\end{align*}
Similarly, 
\begin{align*}
    \max_{\Vert u \Vert_2 = 1}\bigg| \frac{d^2V_\alpha(s)}{d\alpha^2}\Big|_{\alpha = 0} \bigg| &\leq \gamma^2 \Big\Vert M(\alpha)\frac{d\widetilde{P}_{\xi}(\alpha)}{d\alpha}M(\alpha)\frac{d\widetilde{P}(\alpha)}{d\alpha}\xi_\alpha M(\alpha) C\Big\Vert_\infty + \gamma^2 \Big\Vert M(\alpha)\frac{d\widetilde{P}_{\xi}(\alpha)}{d\alpha}M(\alpha)y_\alpha \Big\Vert_\infty \\
    &+ \gamma^2\Big\Vert  M(\alpha)\frac{d\widetilde{P}_{\xi}(\alpha)}{d\alpha}M(\alpha)\lambda_{\alpha}^{P}M(\alpha)\xi_\alpha  \Big\Vert_\infty\\
    &+ \gamma \Big\Vert M(\alpha) \frac{d^2\widetilde{P}(\alpha)}{d\alpha^2}\xi_\alpha M(\alpha)C  \Big\Vert_\infty + \gamma \Big\Vert M(\alpha)\frac{d\widetilde{P}(\alpha)}{d\alpha}\frac{d\xi_\alpha}{d\alpha}  M(\alpha)C \Big\Vert_\infty \\
    &+ \gamma^2 \Big\Vert M(\alpha)\frac{d\widetilde{P}(\alpha)}{d\alpha}\xi_\alpha M(\alpha)\frac{d\widetilde{P}_{\xi}(\alpha)}{d\alpha}M(\alpha)C \Big\Vert_\infty\\
    &+ \gamma \Big\Vert M(\alpha)\frac{d\lambda_\alpha^{P}}{d\alpha}\frac{d\widetilde{P}(\alpha)}{d\alpha}\xi_\alpha \Big\Vert_\infty + \gamma \Big\Vert M(\alpha)\lambda_\alpha^{P}\Big(\frac{d^2\widetilde{P}(\alpha)}{d\alpha^2}\xi_\alpha +  \frac{d\widetilde{P}(\alpha)}{d\alpha}\frac{d\xi_\alpha}{d\alpha}\Big)\Big\Vert_\infty \\
    &+ \gamma\Big\Vert M(\alpha) \frac{d^2\widetilde{P}(\alpha)}{d\alpha^2}y_\alpha\Big\Vert_\infty + \gamma \Big\Vert M(\alpha)\frac{d\widetilde{P}(\alpha)}{d\alpha}\frac{df_\alpha}{d\alpha} \Big\Vert_\infty \\
    &\leq 2\frac{\gamma^2}{(1-\gamma)^3}C_1C_4(C_1C_3 + C_5)C_{max} + \frac{\gamma^2}{(1-\gamma)^2}(C_1C_3 + C_5)\Big(|\mathcal{E}| + |\mathcal{I}|\Big)MC_4 \\
    &\qquad+ \frac{\gamma^2}{(1-\gamma)^3}(C_1C_3 + C_5)C_3C_4 + \frac{\gamma}{(1-\gamma)^2}C_{max}\Big(C_2C_3 + C_1C_5\Big) \\
    &\qquad+ \frac{\gamma}{1-\gamma}\Big(C_1C_3C_6 + C_2C_3C_4 + C_1C_5 + C_2\Big(|\mathcal{E}| + |\mathcal{I}|\Big)MC_4 + C_1C_7\Big)\\
    &= \beta.
\end{align*}
Under direct parameterization, $C_1 \leq \sqrt{|\A|}$ and $C_2 = 0$. The result follows using $\beta$ as the smoothness parameter. 
\end{proof}

% \clearpage 

% \subsection{Function Approximation}\label{sec:function_approximation_convergence}

%\clearpage 

\section{Proofs for Section~\ref{SEC:GRADIENT_ESTIMATION}}\label{sec:appendix_gradient_estimation}
\subsection{Proof of Theorem~\ref{THM:CORRECTION}}% 6.1}
\begin{proof}[Proof of Theorem \ref{thm:correction}]
Define the starting state distribution as $d_0(s) = \mathbbm{1}[s = s_0]$. Then given reweighting $\{\xi_s\}_{s\in \Ss}:\Ss \rightarrow \R$ and transitions $P_\theta$ induced by policy $\pi_\theta$, we establish the relation for any $\gamma \in (0, 1]$ that
\allowdisplaybreaks
\begin{align}
    d_\theta^\xi(s') &= (1-\gamma)\sum_{t=0}^\infty \gamma^t P_{\theta}^\xi(s_t=s') \nonumber \\ 
    &= (1-\gamma)d_0(s') + (1-\gamma)\sum_{t=1}^\infty \gamma^t P_{\theta}^\xi(s_t=s'|s_0) \nonumber \\
    &= (1-\gamma)d_0(s') + \gamma(1-\gamma) \sum_{t=0}^\infty\gamma^t P_{\theta}^\xi(s_{t+1}=s') \nonumber \\
    &= (1-\gamma)d_0(s') + \gamma(1-\gamma) \sum_{t=0}^\infty\gamma^t \sum_s P_\theta(s'|s)\xi_{s}(s')P_{\theta}^\xi(s_t=s) \nonumber \\
    &= (1-\gamma)d_0(s') + \gamma \sum_s P_\theta(s'|s)\xi_{s}(s')\Big((1-\gamma)\sum_{t=0}^\infty\gamma^t P_{\theta}^\xi(s_t=s)\Big)\nonumber \\
    &= (1-\gamma)d_0(s') + \gamma \sum_s P_\theta(s'|s)\xi_{s}(s')d_\theta^\xi(s)\label{eq:d_xi_relation}
\end{align}
Using the same lines of reasoning, for $\gamma \in (0, 1]$
\begin{align}
    d_\theta(s') &= (1-\gamma)d_0(s') + \gamma \sum_s P_\theta(s'|s)d_\theta(s)\label{eq:d_relation}
\end{align}
With this, $d_\theta$ can be seen as the invariant distribution of an induced Markov chain which follows $P_\theta$ with probability $\gamma$ and restarts from the initial state $s_0$ with probability $1-\gamma$. Similarly, $d_\theta^\xi$ can be seen as the invariant distribution of an induced Markov chain which follows $P_\theta\xi$ with probability $\gamma$ and restarts from the initial state $s_0$ with probability $1-\gamma$.\\
\\
%Substituting in $w_\xi(s) := d_\theta^\xi(s)/d_\theta(s)$, 
Multiplying both sides of~\eqref{eq:d_relation} by $f(s')$ and summing over $s'$, we have that 
\begin{equation}\label{eq:d_f_relation}
    \sum_{s'}d_\theta(s')f(s') = (1-\gamma)\sum_{s'}d_0(s')f(s') + \gamma \sum_{s, s'} P_\theta(s'|s)d_\theta(s)f(s'). 
\end{equation}
Similarly, 
\begin{align}
    \sum_{s'}d^\xi_\theta(s')f(s') &= (1-\gamma)\sum_{s'}d_0(s')f(s') + \gamma \sum_{s, s'} P_\theta(s'|s)\xi_s(s')d^\xi_\theta(s)f(s') \nonumber\\ 
    \sum_{s'}w_\xi(s')d_\theta(s')f(s') &= (1-\gamma)\sum_{s'}d_0(s')f(s') + \gamma \sum_{s, s'} P_\theta(s'|s)w_\xi(s')d_\theta(s)f(s')\label{eq:dxi_f_relation}
\end{align}
where the last line uses the definition $w_\xi(s) := d_\theta^\xi(s) / d_\theta(s)$. 
The relation in~\eqref{eq:d_f_relation} is equivalent to 
\begin{align}
    0 &= \E_{(s,s') \sim d_\theta}[\gamma f(s') - f(s)] + (1-\gamma)\E_{s \sim d_0}[f(s)]
    \intertext{Then using $w(s)f(s)$ for $f(s)$, }
    &= \E_{(s,s') \sim d_\theta}[\gamma w(s')f(s') - w(s)f(s)] + (1-\gamma)\E_{s \sim d_0}[w(s)f(s)]. \label{eq:plug}
\end{align}
Plugging~\eqref{eq:plug} into~\eqref{eqn:correction}, 
\begin{align}
    % L(w,f) &= \gamma\E_{(s,s') \sim d_\theta}[(\xi_s(s')w(s) - w(s'))f(s')] - \E_{s' \sim d_\theta}[w(s')f(s')] + (1-\gamma)\E_{s \sim d_0}[(1-w(s))f(s)] \label{eq:plug2}
    L(w,f) &= \gamma\E_{(s,s') \sim d_\theta}[\xi_s(s')w(s)f(s')] - \E_{s \sim d_\theta}[w(s))f(s)] + (1-\gamma)\E_{s \sim d_0}[f(s)] \label{eq:plug2}
\end{align}
Thus when $L(w,f) = 0$, \eqref{eq:plug2} is equivalent to~\eqref{eq:d_f_relation}, which means that $w(s) = w_\xi(s)$.
% \begin{equation}\label{eq:ssd_relation}
%     \sum_{s'}w_\xi(s')d_\theta(s') = \sum_{s'}d_0(s') + \gamma  \sum_{s, s'} P_\theta(s'|s)\xi_{s}(s')w_\xi(s)d_\theta(s). 
% \end{equation}
% Rearranging the definition from~\eqref{eqn:correction}, we have
% \begin{align*}
%      0 &= \gamma\E_{s|s' \sim d_\theta}\Big[w(s)\xi_s(s') - w(s')\Big] + \E_{s \sim d_0}\Big[1 - \frac{w(s)}{1-\gamma}\Big] \\
%      \intertext{Taking the expectation of both sides over $s' \sim d_\theta$, }
%      0 &= \gamma\E_{s, s' \sim d_\theta}\Big[w(s)\xi_s(s') - w(s')\Big] + \E_{s \sim d_\theta}\Big[1 - \frac{w(s)}{1-\gamma}\Big] \\
%      &= \gamma\E_{s, s' \sim d_\theta}\Big[w(s)\xi_s(s')\Big] - \E_{s \sim d_\theta}\Big[w(s)\Big] + \sum_{s'}d_0(s')
% \end{align*}
% where $s, s' \sim d_\theta$ refers to the joint distribution $d_\theta(s, s') = P_\theta(s'|s)d_\theta(s)$. 
% Rearranging this expression gives us~\eqref{eq:ssd_relation}. Thus, $w(s) = w_\xi(s)$ if and only if Theorem \ref{thm:correction} is satisfied. 
\end{proof}

% For a given $s'$, it may be difficult to observe multiple $s$ which result in $s'$. We can instead multiply \eqref{eqn:correction} by a function $f(s')$ and average under $s' \sim d$: 
% \begin{align*}
%     L(w,f) &:= \E_{(s,s') \sim d}\Big[ \Delta(w;s,s')f(s') \Big] \\
%     &= \E_{(s,s') \sim d}\Big[\Big(w(s)\xi(s|s') - w(s')\Big)f(s') \Big]
% \end{align*}
% As in Theorem 1, we have $w \propto w_\xi$ if and only if $L(w,f) = 0$ for any function $f$. We can thus estimate $w_\xi$ with the min-max problem: 
% \begin{equation}
%     \min_w \{ D(w) := \max_{f \in \mathcal{F}}L(w/z_n, f)^2\}
% \end{equation}
% where $\F$ is a set of discriminator functions and $z_w := \E_{s \sim d}[w(s)]$ normalizes $w$ to avoid the trivial solution $w = 0$. In practice, we can approximate $\F$ using an RBF kernel, which for discrete state spaces corresponds to the delta kernel. 

\subsection{Proof of Theorem~\ref{THM:STATIONARITY} }%6.2}
\label{proof:algo_stationary}
First, we state and prove an upper bound on the magnitude of gradients using the bias and variance of a gradient oracle in Theorem \ref{thm:thm_4}. 
Suppose we have a function $f:\R^d\rightarrow\R$ which is differentiable, $L$-Lipschitz, and $\beta$-smooth, with an infimum at $f_*$. Suppose have access to a noisy gradient oracle which returns a vector $\widehat{\nabla}_\theta f(x) \in \R^d$ given a query point $x$. The vector is said to be $\sigma,B$-accurate for parameters $\sigma, B \geq 0$ if for all $x \in \R^d$, the quantity $\delta(x) := \widehat{\nabla}_\theta f(x) - \nabla_\theta f(x)$ satisfies 
\begin{equation}\label{eq:condition}
    \Vert \E[\delta(x)|x]\Vert \leq B \quad\text{and}\quad  \E[\Vert \delta(x)\Vert^2|x] \leq 2(\sigma^2 + B^2)
\end{equation}
Then we have the following guarantee for convergence to a stationary point under stochastic gradient descent with the update $x \leftarrow x - \eta \widehat{\nabla}_\theta f(x)$: 
% \begin{theorem}[\citet{liu2020off}, Theorem 4. \kamyar{why we cite this paper?}]\label{thm:thm_4}
% Suppose $f$ is differentiable and $\beta$-smooth, and the approximate gradient oracle satisfies~\eqref{eq:condition} with parameters $(\sigma_k, B_k)$ at iteration $k$. Then after $K$ iterations, stochastic gradient descent with initial point $x_1$ and stepsize $\eta = 1/\beta$ satisfies 
% \begin{equation}\label{eq:sgd_stationarity}
%     \frac{1}{K}\sum_{k=1}^K\E[\Vert \nabla f(x_k) \Vert^2] \leq \frac{2}{K}\Big(f(x_1) - f_*\Big) + \frac{2}{\beta K}\sum_{k=1}^K\Big(\sigma_k^2 + B_k^2\Big). 
% \end{equation}
% \end{theorem}
\begin{theorem}\label{thm:thm_4}
Suppose $f$ is differentiable, $L$-Lipschitz, and $\beta$-smooth, and the approximate gradient oracle satisfies~\eqref{eq:condition} with parameters $(\sigma_k, B_k)$ for all iterations $k$. Then after $K$ iterations, stochastic gradient descent with initial point $x_1$ and stepsize $\eta = \frac{1}{\beta \sqrt{K}}$ satisfies 
% change 
\begin{equation}\label{eq:sgd_stationarity}
    % \frac{1}{K}\sum_{k=1}^K\E[\Vert \nabla f(x_k) \Vert^2] \leq \frac{2}{K}\Big(f(x_1) - f_*\Big) + \frac{2}{\beta K}\sum_{k=1}^K\Big(\sigma_k^2 + B_k^2\Big). 
    \frac{1}{K} \sum_{k=1}^K \E[ \Vert \nabla f(x_k) \Vert^2 ] \leq  \frac{2\beta}{\sqrt{K}}(f(x_1) - f_*) + \sum_{k=1}^K \frac{2}{K\sqrt{K}} (\sigma_k^2 + B_k^2) + \frac{2L}{K}B_k 
\end{equation}
Additionally, suppose there exists constants $B, \sigma$ such that $B_k \leq B$ and $\sigma_k \leq \sigma$ for all $k$. Then 
\begin{equation}
    \frac{1}{K} \sum_{k=1}^K \E[ \Vert \nabla f(x_k) \Vert^2 ] \leq \frac{2\beta}{\sqrt{K}}(f(x_1) - f_*) + \frac{2}{\sqrt{K}}(\sigma^2 + B^2) + 2LB . 
\end{equation}
\end{theorem}

\begin{proof}
Since $f$ is $\beta$-smooth, we have 
\begin{align*}
    f(x_{k+1}) &\leq f(x_k) + \langle \nabla f(x_k), x_{k+1} - x_k \rangle + \frac{\beta}{2}\Vert x_{k+1} - x_k \Vert^2 \\ 
    &= f(x_k) -  \eta\langle \nabla f(x_k), \widehat{\nabla} f(x_k) \rangle + \frac{\beta\eta^2}{2}\Vert \widehat{\nabla} f(x_k) \Vert^2 \\ &= f(x_k) - \eta\langle \nabla f(x_k),\delta(x_k) + \nabla f(x_k) \rangle + \frac{\beta\eta^2}{2}\Vert \delta(x_k) + \nabla f(x_k) \Vert^2 \\
    &= f(x_k) + \Big(\frac{\beta\eta^2}{2} - \eta\Big)\Vert \nabla f(x_k) \Vert^2 - (\eta - \beta \eta^2) \langle \nabla f(x_k), \delta(x_k)\rangle + \frac{\beta \eta^2}{2} \Vert \delta(x_k)\Vert^2
\end{align*}
where the second step follows from definition of $\widehat{\nabla} f(x)$ and the third step uses the definition of $\delta(x)$. Taking the expectations of both sides and using the properties of~\eqref{eq:condition}, 
\begin{align*}
    \E[f(x_{k+1})] \leq \E[f(x_k)] + \Big(\frac{\beta\eta^2}{2} - \eta\Big)\E[\Vert \nabla f(x_k) \Vert^2] + (\eta - \beta \eta^2)LB_k + \beta\eta^2(\sigma_k^2 + B_k^2), 
\end{align*}
Then summing over iterations $k=1,...K$, 
\begin{align*}
    \E[f(x_{K+1})] \leq f(x_1) + \Big(\frac{\beta\eta^2}{2} - \eta\Big) \sum_{k=1}^K \E[ \Vert \nabla f(x_k) \Vert^2 ] + \sum_{k=1}^K(\eta - \beta \eta^2)LB_k + \beta\eta^2(\sigma_k^2 + B_k^2),
\end{align*}
Using the fact that $f(x_{K+1}) \geq f_*$ and rearranging, 
\begin{align*}
     \Big(\eta - \frac{\beta\eta^2}{2}\Big)\sum_{k=1}^K \E[ \Vert \nabla f(x_k) \Vert^2 ]  &\leq f(x_1) - f_* + \sum_{k=1}^K(\eta - \beta \eta^2)LB_k + \beta\eta^2(\sigma_k^2 + B_k^2),
\end{align*}
Now choosing $\eta =\frac{1}{\sqrt{K}\beta}$, note that $\eta - \beta \eta^2  \leq \frac{1}{\sqrt{K}\beta}$ and $\eta - \frac{1}{2}\beta \eta^2 \geq \frac{1}{2\sqrt{K}\beta}$. Plugging these inequalities in, 
\begin{align*}
    \frac{1}{2\sqrt{K}\beta} \sum_{k=1}^K \E[ \Vert \nabla f(x_k) \Vert^2 ]  &\leq f(x_1) - f_* + \sum_{k=1}^K\frac{L}{\sqrt{K}\beta}B_k + \frac{1}{\beta K}(\sigma_k^2 + B_k^2) \\
    \intertext{Rearranging, }
    \frac{1}{K} \sum_{k=1}^K \E[ \Vert \nabla f(x_k) \Vert^2 ] &\leq  \frac{2\beta}{\sqrt{K}}(f(x_1) - f_*) + \sum_{k=1}^K \frac{2L}{K}B_k + \frac{2}{K\sqrt{K}} (\sigma_k^2 + B_k^2)
\end{align*}
giving the first statement in the theorem. To achieve the second, use the upper bound $B \geq B_k$ and $\sigma \geq \sigma_k$ for all $k$: 
\begin{align*}
    \frac{1}{K} \sum_{k=1}^K \E[ \Vert \nabla f(x_k) \Vert^2 ] &\leq \frac{2\beta}{\sqrt{K}}(f(x_1) - f_*) + \frac{2}{\sqrt{K}}(\sigma^2 + B^2) + 2LB 
\end{align*}
\end{proof}
Next, we use Theorem \ref{thm:thm_4} to prove our convergence result in Theorem \ref{thm:stationarity}: 
\begin{proof}[Proof of Theorem \ref{thm:stationarity}]
We determine $\sigma_k, B_k$ for Algorithm \ref{algo:ssd}. For short, first define $p_\theta(s,a) = \frac{\partial \log \pi_\theta(a|s)}{\partial \theta}$. Then the true gradient $\nabla_\theta V$ is given by 
\begin{equation*}
    \nabla_\theta V = \E_{s,a \sim d_\theta^{\xi_\theta}}\Big[p_\theta(s,a)h_\theta(s,a)\Big] = \E_{d_\theta}\Big[w(s)p_\theta(s,a)h_\theta(s,a)\Big]
\end{equation*}
and the estimated gradient is $\widehat{\nabla}_\theta V$: 
\begin{equation*}
    \widehat{\nabla}_\theta V = \hat{w}(s)p_\theta(s,a)\hat{h}_\theta(s,a)
\end{equation*}
We can bound the bias as 
\begin{align*}
    \Vert \E[\widehat{\nabla}_\theta V - \nabla_\theta V | \theta]\Vert &= \Big\Vert \E_{d_\theta}\Big[w(s)p_\theta(s,a)h_\theta(s,a) -  \hat{w}(s)p_\theta(s,a)\hat{h}_\theta(s,a)\Big] \Big\Vert \\ 
    &\leq \E_{d_\theta}\Big[\Vert (w - \hat{w})p_\theta h_\theta \Vert\Big] + \E_{d_\theta}\Big[\Vert\hat{w}p_\theta \Big(h - h_\theta \Big)\Vert\Big] \\
    &\leq \E_{d_\theta}\Big[|w - \hat{w}|\Vert p_\theta \Vert |h_\theta| \Vert\Big] + \E_{d_\theta}\Big[\hat{w} \Vert p_\theta\Vert |h - h_\theta |\Big] \\ 
    &\leq G\epsilon_w \sqrt{\E_{d_\theta}[h_\theta^2]} + G\epsilon_h\sqrt{\E_{d_\theta}[\hat{w}^2]}
\end{align*}
where the last two inequalities follow from Cauchy-Schwarz, and the last inequality uses Assumption \ref{assum:stationary}. We can upper bound the last term as 
\begin{align*}
    \E_{d_\theta}[\hat{w}^2] &\leq \E_{d_\theta}[w^2] + \E_{d_\theta}[(w - \hat{w}))^2] \leq \sigma_w^2 + \epsilon_w^2
\end{align*}
which gives us 
\begin{equation*}
    \Vert \E[\widehat{\nabla}_\theta V - \nabla_\theta V | \theta]\Vert \leq GC_{max}\epsilon_w + G\epsilon_h\sqrt{\sigma_w^2 + \epsilon_w^2}
\end{equation*}
Similarly, the variance is bounded as 
\begin{align*}
     \E[\Vert\widehat{\nabla}_\theta V - \nabla_\theta V\Vert^2 | \theta] &\leq 2\E_{d_\theta}\Big[\Vert (w - \hat{w})p_\theta h_\theta \Vert^2\Big] + \E_{d_\theta}\Big[\Vert\hat{w}p_\theta \Big(h - h_\theta \Big)\Vert^2\Big] \\
     &\leq 2G^2C_{max}^2\epsilon_w^2 + 2 G^2\epsilon_h^2(\sigma_w^2 + \epsilon_w^2)
\end{align*}
Plugging the bias and variance into Theorem \ref{thm:thm_4} gives the result. 
\end{proof}

\subsection{Experiment Details}\label{appendix:hyperparams}
\paragraph{Hyperparameters}
We use a separate neural network for the policy and critic, both optimized using the Adam optimizer. For solving the optimization problems \eqref{eqn:mcr_value_lagrang} and \eqref{eqn:correction}, we use CVXPY \citep{diamond2016cvxpy,agrawal2018rewriting}. For our experiments, we used the hyperparameters in Table \ref{table:cliffwalk_params}. 
\begin{table}[!htbp]
    \centering
    \begin{tabular}{ll}
        \textbf{Hyperparameters} &  \\ 
        \toprule 
        $\gamma$ & 1.0 \\ 
        learning rate (actor) & 0.001 \\
        learning rate (critic) & 0.001 \\
        batch size (actor) & all \\
        batch size (critic) & 512 \\
        number of iterations (actor) & 1 \\
        number of iterations (critic) & 10 \\
        number of trajectories & 200\\
        \bottomrule
    \end{tabular}
    \caption{Hyperparameters for CliffWalk experiments.}
    \label{table:cliffwalk_params}
\end{table}

\paragraph{Learned State Distribution Correction. }
The learned state distribution correction $w(s)$ for the $\alpha=0.25$ policy in the stochastic Cliffwalk from Figure~\ref{fig:cliffwalk} is shown in Figure~\ref{fig:learned_w} for each of the $4\times 12$ Cliffwalk states. Initially, high state correction weights are given to states with high $\xi$ weights (Figure~\ref{fig:learned_xi}). After training converges to the deterministic policy within 5K episodes, the agent takes only a single path along the topmost squares of the Cliffwalk. As a result, each of the states along this path has equal weight $w(s) = 1/(\text{length of path})$, and the interior states that the policy does not travel to have $w(s) = 0$. 

\begin{figure}[!htp]
\hspace*{-2cm}    
    \centering
    \includegraphics[width=1.25\linewidth]{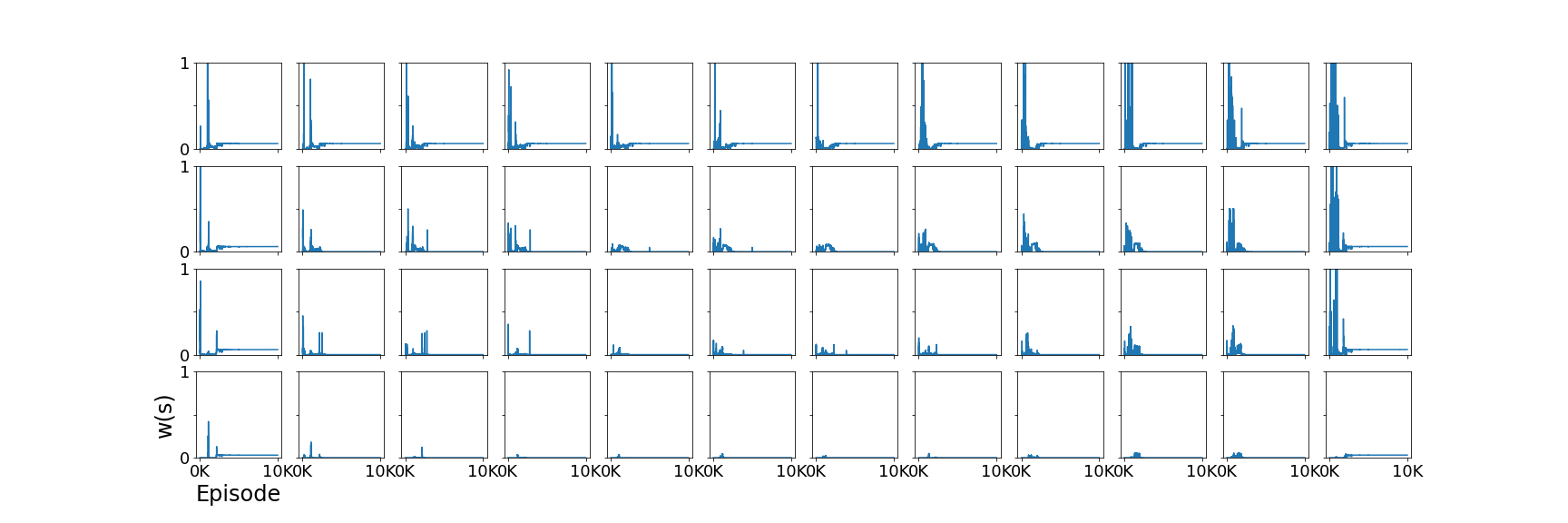}
    \caption{Learned state distribution correction $w(s)$ over 10k training episodes from the $\alpha = 0.25$ policy in the stochastic Cliffwalk in Figure~\ref{fig:cliffwalk}. Each square corresponds to one state in the grid, e.g. the start state is on the bottom left.  }
    \label{fig:learned_w}
\end{figure}

%\clearpage
\paragraph{Learned Lagrangian Variables.}
The $\xi_s(s')$ for the $\alpha=0.25$ policy in the stochastic Cliffwalk from Figure~\ref{fig:cliffwalk} is shown in Figure~\ref{fig:learned_xi} for each of the $4\times 12$ Cliffwalk states. For the start state (bottom left grid), initially $\xi_{s_0}(s_\text{cliff})$ corresponding to walking right and falling into the cliff is high. 
After the policy learns a deterministic path around 3-4K episodes, the $\xi_s(s')$ corresponding to the deterministic transition $s'$ such that $P_\theta(s'|s) = 1$ for each $s$ must also be 1 due to the constraint that $E_{P_\theta}[\xi] = 1$. For $s'$ where $P_\theta(s'|s) = 0$, the $\xi_s(s')$ are free variables and do not affect the value.  

\begin{figure}[!htp]
\hspace*{-2cm}    
    \centering
    \includegraphics[width=1.25\linewidth]{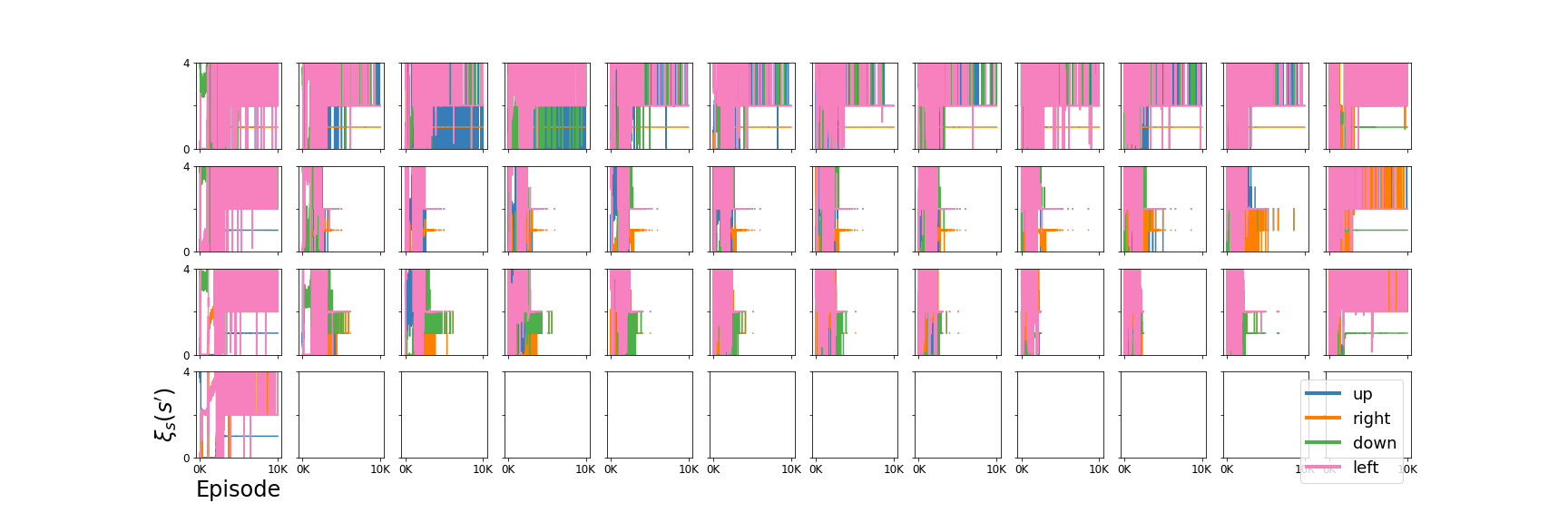}
    \caption{Learned state distribution correction $\xi_s(s')$ over 10k training episodes from the $\alpha = 0.25$ policy in the stochastic Cliffwalk in Figure~\ref{fig:cliffwalk}. Each square corresponds to one state in the grid, e.g. the start state is on the bottom left. For each state $s$, the color of the line corresponds to next state $s'$ resulting from each of the 4 actions taken (see legend). }
    \label{fig:learned_xi}
\end{figure}
%\clearpage

\subsection{Function Approximation for \texorpdfstring{$w, \xi, \lambda$}{w,x,l}}\label{sec:ssd_function}
Finally, in Algorithm \ref{algo:ssd_function} we present a method for learning $w, \xi, \lambda$ in the function approximation setting, such as large or continuous state spaces, where the use of convex solvers may become difficult. 
In such cases, $w$ and $\xi, \lambda$ can be learned using neural networks for a total of five networks--the actor, critic, $w$, $\xi$, and $\lambda$. For $\cvara$, for example, only $\lambda^P$ needs to be learned since the constraint that $\xi \in [0, \frac{1}{\alpha}]$ can naturally be implemented by using a Softmax activation in the final layer of the $\xi$ network, and multiplying the output by $\frac{1}{\alpha}$. The $\lambda^P$ network controls the constraint that $\E_{P_\theta}[\xi] = 1$. \\
\\
The $w$ network seeks to minimize $L(w,f)$ in Theorem~\ref{thm:correction}, an algorithm for which is given in Algorithm 2 of \cite{liu2018breaking}. In practice, for discrete state spaces the delta kernel can be used for $f$, and in continuous state spaces an RBF kernel with radius set to be the median of distances between states in the batch can be used for $f$ \citep{liu2020off}. The $\xi, \lambda$ networks seek to maximize and minimize the MCR objective~\eqref{eqn:mcr_gradient}, respectively. The gradients of the MCR objective with respect to $\xi$ and $\lambda$ are given in Lemma \ref{lem:lagrang_gradient}. 
We recommend that a higher learning rate is using for the $\lambda$ network and a lower learning rate is used for the $\xi$ network. 
The Lagrangian can be cast as a two-player zero sum game with a bilinear term between $\xi$ and $\lambda$, and min-max gradient descent for such objectives has been proven to converge only with finite timescale separation \citep{fiez2020gradient}. 
We do not include experiments as it is beyond the scope of our paper, but hope that algorithms such as Algorithm \ref{algo:ssd_function} can serve as a basis for future work in scaling coherent risk policy gradient algorithms to larger or continuous state spaces. 
\begin{algorithm}[!htb]
\SetAlgoLined
\KwIn{Risk envelope $\U$, learning rates $\{\eta_w, \eta_c, \eta_\xi, \eta_\lambda, \eta \}$, batch sizes $\{B_w, B_c, B_l\}$, updates $\{N_w, N_c, N_l\}$. }
\textbf{Initialize: } Policy network $\theta$, critic network $\theta_c$, $w$ network $\theta_w$, $\xi$ network $\phi$, $\lambda$ network $\psi$\; 
\For{episode $k=1...K$ }{
Generate $N$ trajectories $\tau_k = \{s_0^{(n)}, a_0^{(n)}, c_0^{(n)}, ...\}_{n = 1}^N$ using $\pi_{\theta}$\; 
Pad $\tau_k$ if necessary\;
    \For{state ratio updates $i=1...N_w$}{
        Sample a minibatch $B_w \sim \tau_k$\; 
        Perform one update to $\theta_w$ according to Algorithm~\ref{algo:w_update} with stepsize $\eta_w$\; 
    }
    \For{critic updates $i=1...N_c$}{
        Sample a minibatch $B_c \sim \tau_k$\; 
        Calculate $l_{\theta_c} = \frac{1}{|B_c|}\sum_{(s,s') \sim B_c}\bigg(V_{\theta_c}(s) - \Big(C(s) + L_{\theta, s}(\xi_\phi, \lambda_{\psi})\Big)\bigg)^2$ with $L_{\theta,s}$ from~\eqref{eqn:mcr_value_lagrang}\;
        $\theta_c \leftarrow \theta_c + \eta_c \frac{\partial l_{\theta_c}}{\partial \theta_c}$\;
    }
    \For{$\xi, \lambda$ updates $i=1...N_l$}{
        Sample a minibatch $B_l \sim \tau_k$\; 
        Calculate gradients $\frac{\partial V}{\partial \theta_\xi}$ and $\frac{\partial V}{\partial \theta_\xi}$ according to Lemma \ref{lem:lagrang_gradient}\;
        $\theta_\xi \leftarrow \theta_\xi + \eta_\xi  \frac{\partial V}{\partial \theta_\xi}$\;
        $\theta_\lambda \leftarrow \theta_\lambda + \eta_\lambda  \frac{\partial V}{\partial \theta_\lambda}$\;
    }
    Compute $\widehat h(s,a)$ in \eqref{eqn:mcr_h_theta}\;
    Compute $\nabla_\theta V_\theta = \sum_{s,a,s' \sim \tau_k} w_{\theta_w}(s) \nabla_\theta \log \pi_{\theta_k}(a|s)\widehat{h}(s,a)$\;
    Update $\theta \leftarrow \theta - \eta \nabla_\theta V_\theta$\;
    }
\caption{Stationary State Distribution Reweighting Actor-Critic with Function Approximation }\label{algo:ssd_function}
\end{algorithm}

\clearpage
\begin{algorithm}[!htb]
\SetAlgoLined
\KwIn{$w$ network parameters $\theta_w$, learning rate $\eta_w$, batch of transitions $B_w$, weightings $\xi$, discount factor $\gamma \in (0, 1]$, kernel $k$. }
    \textbf{Compute} $\widehat{L}(w) = \frac{1}{|B_w|}\sum_{i, j \in B_w} \Delta(w; s_i, s_i')\Delta(w; s_j, s_j')k(s_i', s_j')$\; 
    where $\Delta(w, s, s') = w(s)\xi_s(s') - w(s')$.\;
    \textbf{Update} the parameter by $\theta_w \leftarrow \theta_w - \eta_w \nabla_\theta \widehat{L}(w_{\theta_w} / z_{w_{\theta_w}})$\;
    where $z_{w_{\theta_w}}$ is a normalization constant $z_{w_{\theta_w}} = \frac{1}{|B_w|}\sum_{i \in B_w}w_{\theta_w}(s_i)$. \;
\caption{Single $w$ Update (adapted from Algorithm 2 in \citet{liu2018breaking})}
\label{algo:w_update}
\end{algorithm}

\begin{lemma}[Gradient of MCR w.r.t $\xi, \lambda$. ]\label{lem:lagrang_gradient}
Let the policy network be parameterized by $\theta$, the $\xi$-network parameterized by $\phi$, and the $\lambda$-network be parameterized by $\psi$. Then the gradient of the objective $V_{\theta,\phi,\psi}$ with respect to $\phi$ is 
\begin{equation}
    \nabla_\phi V_{\theta,\phi,\psi}(s) = \E_{\xi_\phi}\bigg[\sum_{t=0}^\infty \gamma^t h_\phi(s_t) |s_0 = s; \pi_\theta \bigg]
\end{equation}
where 
\begin{equation}
    h_\phi(s) = \gamma \sum_{s'}P_\theta(s'|s)\nabla_\phi \xi_\phi(s'|s)\bigg(V_{\theta,\phi,\psi}(s') - \lambda_\psi(p|s) - \sum_e \lambda_\psi(e|s)\frac{dg_e(\xi_\phi, P_\theta)}{d\xi(s')} - \sum_i \lambda_\psi(i|s)\frac{df_i(\xi_\phi, P_\theta)}{d\xi(s')}\bigg)
\end{equation}
The gradient with respect to $\psi$ is 
\begin{equation}
    \nabla_\psi V_{\theta,\phi,\psi}(s) = \E_{\xi_\phi}\bigg[\sum_{t=0}^\infty \gamma^t h_\psi(s_t) |s_0 = s; \pi_\theta \bigg]
\end{equation}
where 
\begin{equation}
    h_\psi(s) = -\gamma \bigg(\nabla_\psi \lambda_\psi(p|s)\Big(\sum_{s_1}P_\theta(s'|s)\xi_\psi(s'|s) - 1 \Big) + \sum_e \nabla_\psi \lambda_\psi(e|s) g_e(\xi_\phi, P_\theta) + \sum_i \nabla_\psi \lambda_\psi(i|s) f_i(\xi_\phi, P_\theta)\bigg)
\end{equation}
\end{lemma}
\begin{proof}
The proof follows from applying the chain rule to the Lagrangian formulation of the objective~\eqref{eqn:mcr_value_lagrang} using the same techniques as the proof of Theorem~\ref{thm:mcrp_gradient} in Appendix~\ref{sec:appendix_value}. 
\end{proof}

%\clearpage 

\end{document}